%% file: main.tex
\documentclass{article}
\usepackage{geometry}

\usepackage{microtype}
\usepackage{graphicx}
\usepackage{subfigure}
\usepackage{booktabs} 

\usepackage{mathtools}
\usepackage{pdflscape}
\usepackage{listings}
\usepackage{amsfonts}
\usepackage{caption}
\usepackage{nicefrac}
\usepackage{amsthm}
\usepackage{amsmath}
\usepackage{dsfont}
\usepackage{color}

\usepackage{algorithm}
\usepackage{algorithmic}

\newcommand{\E}{\mathbb{E}}
\newcommand{\R}{\mathbb{R}}
\newcommand{\N}{\mathbb{N}}
\newcommand{\1}{\mathds{1}}
\newcommand{\x}{\mathbf{x}}
\newcommand{\y}{\mathbf{y}}
\newcommand{\e}{\mathbf{e}}

\newcommand{\cT}{\mathcal{T}}

\newcommand{\cX}{\mathcal{X}}

\newcommand{\proj}{\mathrm{proj}}
\newcommand{\ip}[1] {\left\langle #1 \right\rangle }
\newcommand{\bp}{\breve{p}}

\newcommand{\cF}{\mathcal{F}}

\DeclareMathOperator*{\argmin}{arg\,min}

\newtheorem{definition}{Definition}
\newtheorem{theorem}{Theorem}
\newtheorem{lemma}{Lemma}

\usepackage{hyperref}

\usepackage{authblk}
\title{Stochastic Coordinate Minimization with Progressive Precision for Stochastic Convex Optimization}
\author[$\ddagger$]{Sudeep Salgia}
\author[$\ddagger$]{Qing Zhao}
\author[*]{Sattar Vakili}
\affil[$\ddagger$]{School of Electrical \& Computer Engineering, Cornell University, Ithaca, NY, \emph{\{ss3827,qz16\}@cornell.edu} }
\affil[*]{Prowler.io, Cambridge, UK, \emph{sattar@prowler.io}}

\date{}

\usepackage[redeflists]{IEEEtrantools}

\begin{document}

\maketitle

\begin{abstract}
A framework based on iterative coordinate minimization (CM) is developed for stochastic convex optimization.
Given that exact coordinate minimization is impossible due to the unknown stochastic nature of the objective function, the crux of the proposed optimization algorithm is an optimal control of the minimization precision in each iteration. 
We establish the optimal precision control and the resulting order-optimal regret performance for strongly convex and separably nonsmooth functions. 
An interesting finding is that the optimal progression of precision across iterations is independent of the low-dimensional CM routine employed, suggesting a general framework for extending low-dimensional optimization routines to high-dimensional problems. The proposed algorithm is amenable to online implementation 
and inherits the scalability and parallelizability  properties of CM for large-scale optimization. Requiring only a sublinear order of message exchanges,  it also lends itself well to distributed computing as compared with the alternative approach of coordinate gradient descent.

\end{abstract}

\input{introduction.tex}
\input{problem_formulation.tex}
\input{basic_structure.tex}

\input{optimal_design.tex}

\input{termination_conditions.tex}

\input{extensions_discussions.tex}
\input{simulation.tex}

\bibliography{citations}
\bibliographystyle{IEEEtran}

\input{Appendix.tex}

\end{document}

%% file: introduction.tex
\section{Introduction} 
\label{sec:introduction}

\subsection{Stochastic Convex Optimization}

Stochastic convex optimization aims at minimizing a random loss function $F(\x;\xi)$ in expectation:
\begin{equation}
    f(\x) = \E_\xi \left[ F(\x;\xi)\right],
    \label{eq:SCO}
\end{equation}
where $\x$ is the decision variable in a convex and compact set $ \mathcal{X}\subset \R^d$ and $\xi$ an endogenous random vector. The probabilistic model of $\xi$ is unknown, or even when it is known, the expectation of $F(\x;\xi)$ over $\xi$ cannot be analytically characterized. As a result, the objective function $f(\x)$ is unknown. \\

With the objective function unknown, the decision maker can only take a trial-and-error learning approach by choosing, sequentially in time, a sequence of query points $\{\x_t\}_{t=1}^T$ with the hope that the decisions improve over time. Various error feedback models have been considered. The zeroth-order vs. first-order feedback pertains to whether the random loss $F(\x_t;\xi_t)$ or its gradient $G(\x_t;\xi_t)$ at each query point $\x_t$ is used in the learning algorithm. The full-information vs. bandit feedback relates to whether the entire loss function $F(\x;\xi_t)$ over all $\x$ or only the random loss or gradient at the queried point $\x_t$ is revealed at each time. \\

The performance measure has traditionally focused on the convergence of $\x_T$ to the minimizer $\displaystyle \x^*=\arg\min_{\x\in\mathcal{X}}f(\x)$ or $f(\x_T)$ to $f(\x^*)$. In an online setting, a more suitable
performance measure is the cumulative regret defined as the expected cumulative loss at the query points in excess to the minimum loss 
\begin{align}
    R(T) = \E\left[\sum_{t=1}^T ( F(\x_t;\xi_t)-f(\x^*)) \right]
\end{align}
This performance measure gives rise to the exploration-exploitation tradeoff: the need to explore the entire domain $\mathcal{X}$ for the sake of future decisions and the desire to exploit the currently best decision indicated by past observations to reduce present loss.  Regret $R(T)$ 
is a finer measure than the convergence of the final decision $\x_T$.
A learning algorithm with a sublinear regret order in $T$ implies
the convergence of $f(\x_T )$ to $f(\x^*)$, and the specific order measures the rate of convergence. \\

The archetypal statistical learning problem of classification based on random instances is a stochastic optimization problem, where the decision variable $\x$ is the classifier and $\xi$ the random instance consisting of its feature vector and hidden label. The probabilistic dependence between feature and label is unknown. Another example is the design of large-scale complex physical systems that defy analytical modeling. Noisy observations via stochastic simulation is all that is available for decision making. 

\subsection{From SGD to SCD}
\label{subsec:sgd_to_scd}

Stochastic convex optimization was pioneered by Robbins and Monro in 1951~\cite{Robbins1951}, who studied the problem of approximating the root of a monotone function
$g(\x)$ based on successive observations of noisy function
values at chosen query points. The problem was originally referred to as stochastic approximation, later also known as stochastic root
finding~\cite{Pasupathy2011}. Its equivalence to the first-order
stochastic convex optimization is immediate when $g(\x)$
is viewed as the gradient of a convex function $f(\x)$ to be minimized. The stochastic gradient descent (SGD) approach developed
by Robbins and Monro~\cite{Robbins1951} has long become a classic and
is widely used. The basic idea of SGD is to choose the next
query point $\x_{t+1}$ in the opposite direction of the observed
gradient while ensuring $\x_{t+1}\in \mathcal{X}$ via a projection operation. Numerous variants of SGD with improved performance have since been developed and their performance analyzed under various measures (See~\cite{Ruder2016, Bottou2018} for recent surveys). \\

The high cost in computing full gradients in large-scale high-dimensional problems and the resistance of SGD to parallel and distributed implementation have prompted the search for alternative approaches that enjoy better scalability and parallelizability. \\

A natural choice is iterative coordinate minimization (CM)  that has been widely used and analyzed for optimizing a known deterministic function~\cite{Wright2015}. Also known as alternating minimization, CM is rooted in the methodology of decomposing high-dimensional problems into a sequence of simpler low-dimensional ones. Specifically, CM-based algorithms approach the global minimizer by moving successively to the minimizer in each coordinate\footnote{We use the term ``coordinate'' to also refer to a block of coordinates.} while keeping other coordinates fixed to their most recent values. For known deterministic objective functions, it is often assumed that the minimizer in each coordinate can be computed, and hence attained in each iteration. \\

When coordinate-wise minimization is difficult to carry out, coordinate (gradient) descent (CD) can be employed, which takes a single step (or a fixed number of steps) of (gradient) descent along one coordinate and then moves to the next coordinate\footnote{The term coordinate descent is often used to include coordinate minimization. We make an explicit distinction between CD and CM in this paper. The former refers to taking a single step (or a pre-fixed number steps) of (gradient) descent along one coordinate and then move to another coordinate. The latter moves along each coordinate with the specific goal of arriving at the minimizer (or a small neighborhood) in this coordinate before switching to another coordinate.}. For quadratic objective functions, CD with properly chosen step sizes essentially carries out coordinate minimization. For general objective functions, however, it is commonly observed that CM outperforms CD~\cite{Wright2015, Beck2013, Tibshirani2013}. \\

While CD/CM-based algorithms have been extensively studied for optimizing deterministic functions, their extensions and resulting performance for stochastic optimization are much less explored. CD can be applied to stochastic optimization with little modification. Since the noisy partial gradient along a coordinate can be viewed as an estimate of the full gradient, stochastic coordinate descent (SCD) has little conceptual difference from SGD.  In particular, when the coordinate is chosen uniformly at random at each time, the noisy partial gradient along the randomly chosen coordinate is an unbiased estimate of the full gradient. All analyses of the performance of SGD directly apply. More sophisticated extensions of CD-based methods have been developed in a couple of recent studies (see Sec.~\ref{subsec:related_work}).   \\

Since exact minimization along a coordinate is impossible due to the unknown and stochastic nature of the objective function, the extension of CM to stochastic optimization is much less clear.  This appears to be a direction that has not been taken in the literature and is the focus of this work. 

\subsection{Main Results}

While both CD- and CM-based methods enjoy scalability and parallelizability, CM often offers better empirical performance and has a much lower overhead in message exchange in distributed computing (due to its sublinear order of switching across coordinates in comparison to the linear order in CD). It is thus desirable to extend these advantages of CM to stochastic optimization. \\

In this paper, we study stochastic coordinate minimization for stochastic convex optimization. We develop a general framework for extending any given low-dimensional optimization algorithm to high-dimensional problems while preserving its level of consistency and regret order. Given that exact minimization along coordinates is impossible, the crux of the proposed framework---referred to as Progressive Coordinate Minimization (PCM)---is an optimal control of the minimization precision in each iteration. Specifically, a PCM algorithm is given by a tuple $\left( \{\epsilon_k\}, \upsilon, \tau \right)$, where $\{\epsilon_k\}_{k \in \N}$ governs the progressive precision of each CM iteration indexed by $k$, $\upsilon$ is an arbitrary low-dimension optimization routine employed for coordinate minimization, and $\tau$ is the self-termination rule for stopping $\upsilon$ at the given precision $\epsilon_k$ in each iteration $k$. We establish the optimal precision control and the resulting order-optimal regret performance for strongly convex and separably non-smooth functions. An interesting finding is that the optimal progression of precision across iterations is independent of the low-dimension routine $\upsilon$, suggesting the generality of the framework for extending low-dimension optimization algorithms to high-dimensional problems. \\

We also illustrate the construction of order-optimal termination rules for two specific optimization routines: SGD (applied to minimize along a coordinate) and RWT (recently proposed in~\cite{Vakili2019a, Vakili2019b}). While SGD is directly applicable to high-dimensional problems, its extension within the PCM framework leads to a marriage between the efficiency of SGD with the scalability and parallelizability  of CM. RWT as proposed in~\cite{Vakili2019a, Vakili2019b} is only applicable to one-dimensional problems. With no hyper-parameters to tune, however, it has an edge over SGD in terms of robustness and self-adaptivity to unknown function characteristics. For both low-dimensional routines, we demonstrate their high-dimensional extensions within the PCM framework.  Empirical experiments using the MNIST dataset show superior performance of PCM over SCD, which echoes the comparison between CM and CD in deterministic settings.

\subsection{Related Work}
\label{subsec:related_work}

CD/CM-based methods for optimizing a known deterministic function have a long history. While such methods had often been eclipsed by more high-performing algorithms, they have started to enjoy increasing popularity in recent years due to the shifted needs from high accuracy to low cost, scalability, and parallelizability in modern machine learning and data analytics applications~\cite{Hsieh2008a, Hsieh2008b, Nesterov2012, Nesterov2014, Richtarik2016a, Richtarik2016b}.   \cite{Wright2015, Fercoq2019} provide a rather detailed literature survey with insights on the development of CD/CM methods over the years. \\

Early studies on the convergence of CM-based approaches include \cite{Luo1992, Tseng2001, Tseng2008, Tseng2009, Saha2013}. CD-based methods have proven to be easier to analyze, especially under the setup of randomized selection of coordinates \cite{Nesterov2012, Leventhal2010, Tewari2011, Tao2012, Deng2013, ShalevShwartz14, Csiba2015, Karimi2016, Salehi2018}. Such CD/CM-based algorithms are often referred to as stochastic CD/CM in the literature due to the randomly chosen coordinates. Optimizing a known deterministic function, however, they are fundamentally different from the stochastic optimization algorithms considered in this work. The term CD/CM with random coordinate selection as used in~\cite{Richtarik2014, Lu2015} gives a more accurate description. \\

CD-based methods have been extended to mini batch settings or for general stochastic optimization problems \cite{Razaviyayn2013, Wang2014, Zhao2014, Dang2015, Reddi2015, Xu2015, Zhang2016, Konecny2016}. In particular, \cite{Dang2015} extended block mirror descent to stochastic optimization. Relying on an averaging of decision points over the entire horizon to combat stochasticity, this algorithm is not applicable to online settings and does not seem to render tractable regret analysis. \cite{Wang2014} gave an online implementation of SCD, which we compare with in Sec.~\ref{sec:simulations}.  \\

The progressive precision control in the framework developed in this work bears similarity with inexact coordinate minimization that has been studied in the deterministic setting (see, for example, \cite{Deng2013, Razaviyayn2013, Grippo1999, Tappenden2016}).
The motivation for inexact minimization in these studies is to reduce the complexity of the one-dimensional optimization problem, which is fundamentally different from the root cause arising from the unknown and stochastic nature of the objective function. The techniques involved hence are inherently different with different design criteria.


%% file: problem_formulation.tex
\section{Problem Formulation} 
\label{sec:problem_formulation}

We consider first-order stochastic convex optimization with bandit feedback. The objective function $f(\x)$ over a convex and compact set $\cX\subset \R^d$ is unknown and stochastic as given in~\eqref{eq:SCO}. Let $g(\x) \equiv \nabla f(\x)$ be the (sub)gradient of $f(\x)$. Let $G(\x; \xi)$ denote unbiased gradient estimates satisfying $\E_\xi [G(\x; \xi)]=g(\x)$. Let $g_i(\x)$ (similarly, $G_i(\x; \xi)$) denote the partial (random) gradient along the $i$-th coordinate  ($i=1,\ldots, d$). Let $\x_i$ and $\x_{-i}$ denote, respectively, the $i$-th element and the $(d-1)$ elements other than the $i$-th element of $\x$. We point out that while we focus on coordinate-wise decomposition of the function domain, extension to a general block structure is straightforward.  

\subsection{The Objective Function and the Noise}

We consider objective functions that are convex and possibly non-smooth with the following composite form:
\begin{equation}
f(\x) = \psi(\x) + \phi(\x),    
\end{equation}
where $\phi(\x)$ is a coordinate-wise separable convex function (possibly non-smooth) of the form $\phi(\x) = \sum_{ i =1}^{d} \phi_i(\x_i)$ for some one-dimensional functions $\{\phi_i(x),\, x\in \cX_i\}_{i=1}^d$ and $\psi$ is $\alpha$-strongly convex and $\beta$-smooth. More specifically, for all $\x, \y \in \cX$ 
	\begin{gather}
	 	\psi(\y) \geq \psi(\x) + \ip{\nabla \psi(\x), \y - \x} + \frac{\alpha}{2} \|\y - \x\|^2_2 \\
	 	\|\nabla\psi(\x) - \nabla\psi(\y)\|  \leq \beta \| \x - \y \|_2 
	 \end{gather} 
Let $\cF_{\alpha, \beta}$ denote the set of all such functions. \\

The above composite form of the objective function has been widely adopted in the literature on CM and CD~\cite{Wright2015}. The separably non-smooth component $\phi$ arises naturally in many machine learning problems that often involve separable regularization such as $\ell_1$ norm and box constraints. \\   

Next we specify the probabilistic model of the noisy partial gradient estimates given by the distributions of the zero-mean random variables $\{G_i(\x;\xi)-g_i(\x)\}_{i=1}^d$. The distribution of  $G_i(\x;\xi)-g_i(\x)$ is said to be sub-Gaussian if its moment generating function is upper bounded by that of a Gaussian with variance $\sigma_i^2$. This implies that for all $s \in \R$ and $i = 1,2, \dots d$, we have
\begin{align}
    \E_{\xi} \left [ \exp( s [G_i(\x; \xi) - g_i(\x; \xi)])\right] \leq \exp \left(\frac{s^2\sigma_i^2}{2}\right)
\end{align}

We also consider heavy-tailed distributions where the $b^{\text{th}}$ raw moment is assumed to be bounded for some $b \in (1, 2)$. Note that this includes distributions with unbounded variance.

\subsection{Consistency and Efficiency Measures}

At each time $t$, the decision maker chooses a query point $\x_t$ and a coordinate $i_t$. Subsequently,  an immediate loss $F(\x_t;\xi_t)$ is incurred, and a random gradient along the $i_t^{\text{th}}$ coordinate is observed. An optimization algorithm $\Upsilon=\{\Upsilon_t\}_{t=1}^T$ is a sequence of  mappings from past actions and observations to the next choice of query point and coordinate. The performance of $\Upsilon$ is measured by the cumulative regret defined as
\begin{align}
    {R}_{\Upsilon}(T) = \E \left[ \sum_{t = 1}^T F(\x_t, \xi_t) - F(\x^*, \xi_t)\right]
\end{align}
where the expectation is with respect to the random process of the query points and gradient observations induced by the algorithm $\Upsilon$ under the i.i.d. endogenous process of $\{\xi_t\}_{t=1}^T$. \\

In general, the performance of an algorithm depends on the underlying unknown objective function $f$ (a dependency omitted in the regret notation for simplicity). Consider, for example, an algorithm that simply chooses one function in $\cF_{\alpha, \beta}$ and sets its query points $\x_t$ to the minimizer of this function for all $t$ would perform perfectly for the chosen function but suffers a linear regret order for all objective functions with sufficient deviation from the chosen one. It goes without saying that such heavily biased algorithms that completely forgo learning are of little interest. \\

We are interested in algorithms that offer good performance for all functions in $\cF_{\alpha, \beta}$. An algorithm $\Upsilon$ is \emph{consistent} if for all $f\in \cF_{\alpha, \beta}$, the end point $\x_T$ produced by $\Upsilon$ satisfies 
\begin{equation}
    \lim_{T\rightarrow \infty} \E[f(\x_T)] = f(\x^*).
\end{equation}
A consistent algorithm offers a sublinear regret order. This is also known as Hannan consistency or no-regret learning \cite{Hannan1957}. The latter term makes explicit the diminishing behavior of the average regret per action. \\

To measure the convergence rate of an algorithm, we introduce the concept of \emph{$p$-consistency}. For a parameter $p\in (0,1)$, we say $\Upsilon$ is \emph{$p$-consistent} if
\begin{equation}
\sup_{f\in \cF_{\alpha, \beta}} \left(\E[f(\x_T)] - f(\x^*)\right) \sim \Theta(T^{-p}). \label{eq:p_consistency_def}
\end{equation}
A $p$-consistent algorithm offers an $O(T^{1-p})$ regret order for all  $f\in \cF_{\alpha, \beta}$. The parameter $p$ measures the convergence rate. \\

An \emph{efficient} algorithm is one that achieves the optimal convergence rate, hence lowest regret order. Specifically, $\Upsilon$ is \emph{efficient} if for all initial query points $\x_1\in \cX$, the end point $\x_T$ produced by $\Upsilon$ satisfies, for some $\lambda >0$,   
    \begin{align}
		  \sup_{f \in \cF_{\alpha, \beta}} \left(\E[f(\x_T)] - f(\x^*)\right) \sim  (f(\x_1) - f(\x^*))^{\lambda} \Theta(T^{-1}).   \label{eq:eff_algo_def}
	\end{align}
An efficient algorithm offers the optimal $\log T$ regret order for all  $f\in \cF_{\alpha, \beta}$. In addition, it is able to leverage favorable initial conditions when they occur. We note here that the specific value of $\lambda$ affects only the leading constant, but not the regret order. Hence for simplicity, we often use $p$-consistency with $p=1$ to refer to efficient algorithms.


%% file: basic_structure.tex
\section{Progressive Coordinate Minimization} 
\label{sec:the_basic_structure_of_pcm}

In this section, we present the PCM framework for extending low-dimensional optimization routines to high-dimensional problems. After specifying the general structure of PCM, we lay out the optimality criteria for designing its constituent components. 

\subsection{The General Structure of PCM}

Within the PCM framework, an algorithm is given by a tuple $\Upsilon\left( \{\epsilon_k\}, \upsilon, \tau \right)$, where $\{\epsilon_k\}_{k \in \N}$ governs the progressive precision of each CM iteration indexed by $k$, $\upsilon$ is the low-dimension optimization routine employed for coordinate minimization, and $\tau$ is the self-termination rule (i.e., a stopping time) for stopping $\upsilon$ at the given precision $\epsilon_k$ in each iteration $k$. Let $\tau(\epsilon)$ denote the (random) stopping time for achieving $\epsilon$-precision under the termination rule $\tau$. \\

A PCM algorithm $\Upsilon\left( \{\epsilon_k\}, \upsilon, \tau \right)$ operates as follows. At $t=1$, an initial query point $\x^{(1)}$ and coordinate $i_1$ are chosen at random. The CM routine $\upsilon$ is then carried out along coordinate $i_1$ with all other coordinates fixed at $\x_{-i_1}^{(1)}$. At time $\tau(\epsilon_1)$, the first CM iteration ends and returns its last query point $\x_{\tau(\epsilon_1),\, i_1}$. The second iteration starts along a coordinate $i_2$ chosen uniformly at random and with the $i_1$ coordinate updated to its new value $\x_{\tau(\epsilon_1),\, i_1}$. The process repeats until the end of horizon $T$ (see Algorithm~1 below). 

\begin{algorithm}
	\caption{PCM $\Upsilon\left( \{\epsilon_k\}, \upsilon, \tau \right)$}
	\label{alg1}
	\begin{algorithmic}
		\STATE {\bfseries Input:} initial point $\x^{(1)}$.
		\STATE Set $k \leftarrow 1$, $t \leftarrow 1$
	    \REPEAT
	    \STATE Choose coordinate $i_{k}$ uniformly at random. 
	    \STATE Carry out $\upsilon$ along the direction $i_{k}$ as follows:
	    \STATE \hspace{1em} Set the initial point to $\x^{(k)}_{i_{k}}$ with fixed $\x^{(k)}_{-i_{k}}$.
	    \STATE \hspace{1em} Continue until $\tau(\epsilon_{k})$.
	    \STATE \hspace{1em} Return the final point $\x_{\tau(\epsilon_{k}), i_{k}}$.
	    \STATE $\x^{(k+1)} \leftarrow \bigg(\x_{\tau(\epsilon_{k}), i_{k}},\, \x^{(k)}_{-i_{k}}\bigg)$
	    \STATE $ k \leftarrow k + 1$
	    \STATE $t \leftarrow t + {\tau}(\epsilon_{k})$
	    \UNTIL{$t = T$}
	\end{algorithmic}
\end{algorithm}

\subsection{Optimal Design of Constituent Components}

PCM presents a framework for extending low-dimensional optimization algorithms to high-dimensional problems. The CM routine $\upsilon$ in a PCM algorithm is thus given, and we allow it to be an arbitrary $p$-consistent algorithm for any $p\in (0,1]$ (note that the definitions of $p$-consistency and efficiency in Sec.~\ref{sec:problem_formulation} apply to arbitrary dimension.) Allowing arbitrary low-dimensional routines make PCM generally applicable, and the inclusion of consistent but not efficient (i.e., $p<1$) routines responds to the shifted needs for low-cost solutions of only modest accuracy, as seen in modern machine learning and data analytics applications.   \\ 

It is readily seen that for every $f(\x) \in \cF_{\alpha, \beta}$, its low-dimension restriction $f(\cdot, \x_{-i})$ for arbitrarily fixed $\x_{-i}$ belongs in $\cF_{\alpha, \beta}$. Consequently, for a given low-dimensional routine $\upsilon$ with a certain consistency/efficiency level $p\in (0, 1]$ (which needs to hold for all low-dimensional restrictions in $\cF_{\alpha, \beta}$; see~\eqref{eq:p_consistency_def},~\eqref{eq:eff_algo_def}), its high-dimensional extension cannot have a better consistency level (or equivalently, lower regret order). The best possible outcome is that the high-dimensional extension preserves the $p$-consistency and the regret order of the low-dimensional algorithm for high-dimensional problems. \\

The design objective of PCM is thus to choose $\{\epsilon_k\}$ and $\tau$ for a given low-dimensional $p$-consistent algorithm $\upsilon$ such that the resulting high-dimensional algorithm preserves the regret order of $\upsilon$. \\

The above optimization can be decoupled into two steps. First, the termination rule $\tau$ is designed to meet an order-optimal criterion as specified below. The optimal design of $\tau$ is specific to the routine $\upsilon$, as one would expect. In the second step, the progression of precision $\{\epsilon_k\}$ is optimized for the given $\upsilon$ augmented with the order-optimal $\tau$ to preserve the $p$-consistency. Quite surprisingly, as shown in Sec.~\ref{sec:optimal_precision_control}, there exists a \emph{universal} optimal $\{\epsilon_k\}$ that is independent of not only the specific routine $\upsilon$ but also the specific consistency value $p\in (0, 1]$. \\

\begin{definition}
 For a given $p$-consistent ($p\in (0, 1]$) low-dimensional algorithm $\upsilon$ and a given $\epsilon >0$, let $\tau(\epsilon)$ denote a stopping time over the random process of $\{x_t\}_{t\ge 1}$ induced by $\upsilon$ that satisfies $\E[f(x_{\tau(\epsilon)})]-f(x^*) \le \epsilon$. A  termination rule $\tau$ is order optimal if for all $\epsilon > 0$, we have
 \begin{equation}
   \sup_{f\in \cF_{\alpha, \beta}}   \E[\tau(\epsilon)] \sim \Theta(\epsilon^{-1/p}).  
 \end{equation}
 \label{def:tau_optimal}
 \end{definition}
Note that the above definition is for the dimensionality as determined by the given algorithm $\upsilon$ with $f$ and $\cF_{\alpha, \beta}$  defined accordingly. \\

An order-optimal termination rule is one that realizes the exponent $p$ of the consistency of the underlying algorithm $\upsilon$. The design of such termination rules is specific to $\upsilon$, which we illustrate in Sec.~\ref{sec:termination_conditions} for two representative efficient low-dimensional routines.

%% file: optimal_design.tex
\section{The Optimal Precision Control} 
\label{sec:optimal_precision_control}

The theorem below establishes the optimal design of $\{\epsilon_k\}$ for arbitrary $p$-consistent low-dimensional routines.
\begin{theorem}
Let $\upsilon$ be an arbitrary $p$-consistent ($p\in (0, 1]$) low-dimensional routine and $\tau$ an order-optimal termination rule. For all $\gamma \in [(1 - \alpha/(d\beta))^{1/2}, 1)$ and $\epsilon_0 > 0$, the 
PCM algorithm $\Upsilon \left(\{\epsilon_0 \gamma^k\}, \upsilon, \tau\right)$ achieves a regret of $O(T^{1-p} \log^p T)$ for all $f\in\cF_{\alpha,\beta}$. 
    \label{thm: PCM_p_consistent}
\end{theorem}

Theorem~\ref{thm: PCM_p_consistent} shows that setting $\epsilon_k = \epsilon_0 \gamma^k$ preserves the regret order\footnote{The preservation of the regret order is exact for efficient routines. For consistent but not efficient (i.e., $p<1$) routines, the preservation is up to a poly-$\log$ term which is dominated by the term of $T^{1-p}$.}    of $\upsilon$. It is thus optimal. Such a choice of $\{\epsilon_k\}$ is independent of $\upsilon$ as well as the consistency level $p$ of $\upsilon$, suggesting a general framework for extending low-dimensional optimization routines to high-dimensional problems. \\

The proof of Theorem~\ref{thm: PCM_p_consistent} is based on a decomposition of the regret as given below.  Let $K$ denote the (random) number of iterations until time $T$.  Let $\x_{i_k}^* = \argmin_{x} f((\x^{(k-1)}_{-i_k}, x))$ be the minimizer in the $i_k^{\text{th}}$ coordinate with other coordinates fixed to $\x^{(k-1)}_{-i_k}$ (i.e., values from the previous iteration). Let $\x^*_{(i_k, \x^{(k-1)})} = (\x^{(k-1)}_{-i_k}, \x_{i_k}^*)$. Let $t_k = t_{k-1} + {\tau}_{\upsilon}(\epsilon_k)$ with $t_0 = 0$ denote the (random) time instants that mark the end of each iteration. We then have
\begin{align}
	 {R}_{\Upsilon}(T) & = \E \left[ \sum_{t = 1}^T F(\x_{t}, \xi_t) - F(\x^*, \xi_t) \right] \nonumber \\
	& = \E \left[ \sum_{k = 1}^K \sum_{t = t_{k-1}+1}^{t_k} F(\x_{t}, \xi_t) - F(\x^*, \xi_t) \right]. \nonumber 
\end{align}
This can be split into two terms using the local minimizer as
\begin{align}
	 {R}_{\Upsilon}(T)  & =   \underbrace{\E \left[ \sum_{k = 1}^K   \sum_{t = t_{k-1}+1}^{t_k} \left[ F(\x_{t}, \xi_t) - F(\x^*_{(i_k, \x^{(k-1)})}) \right] \right]}_{R_1} \nonumber \\
	&   \ \ \ \ \  + \underbrace{\E \left[ \sum_{k = 1}^K \sum_{t = t_{k-1}+1}^{t_k} \left[ F(\x^*_{(i_k, \x^{(k-1)})}) - F(\x^*) \right] \right]}_{R_2}. \label{eq: regret_decomposition}
\end{align}
The first term $R_1$ corresponds to the regret incurred by the low-dimensional routine $\upsilon$ carried out along one dimension. Note that this regret is computed with respect to the one-dimensional local minima $\x^*_{(i_k, \x^{(k-1)})}$. The second term $R_2$ corresponds to the loss incurred at the one-dimensional local minima in excess to the global minimum $\x^*$. \\

The above regret decomposition also provides insight into the optimal design of $\{\epsilon_k\}$. To achieve a low regret order, it is desirable to equalize the orders of $R_1$ and $R_2$. If a more aggressive choice of $\epsilon_k$ is used, then the rate of decay of the CM iterates is unable to compensate for the time required for higher accuracy, resulting in $R_2$ dominating $R_1$. On the other hand, a more conservative choice will lead to a slower decay in objective function with an increased number of CM iterations. This would result in increasing both the terms to an extent where $\Upsilon$ will no longer be able to maintain the consistency level of $\upsilon$. 

\begin{proof}
 We give here a sketch of the proof. The analysis of $R_1$ and $R_2$ builds on analytical characterizations of the following two key quantities: the expected number $\E[K]$ of CM iterations and the convergence rate of CM outputs $\{\x^{(k)}\}_{k=1}^K$. They are given in the following two lemmas.  

\begin{lemma}
	Let $\upsilon$ be a $p$-consistent policy for some $p \in (0,1]$ and $\tau$ its order-optimal termination rule. Under $\Upsilon \left(\{\epsilon_0 \gamma^k\}, \upsilon, \tau\right)$, we have $\E[K]\sim O(\log T)$ for all $f \in \cF_{\alpha, \beta}$. \label{lemma_k_lb}
\end{lemma}


\begin{lemma}
	Let $\{\x^{(k)}\}$ be the sequence of CM outputs generated under $\Upsilon \left(\{\epsilon_0 \gamma^k\}, \upsilon, \tau\right)$ for a function $f \in \cF_{\alpha, \beta}$. Then the sequence of points $\{\x^{(k)}\}$ satisfy $\E[f(\x^{(k)}) - f(\x^*)] \leq F_0 \gamma^k$ for all $k \geq 0$ and for all $\displaystyle \gamma \in [ (1 - \alpha/(d \beta))^{1/2}, 1) $ where $\displaystyle F_0 = \max \left\{ f(\x_0) - f(\x^*), {\epsilon_0}/{ (1 - \gamma) } \right\}$ \label{lemma_CM}.
\end{lemma}


$R_1$ is bounded using the consistency level of the routine $\upsilon$ augmented with the termination rule $\tau$. $R_2$ is bounded using Lemma~\ref{lemma_CM} and the expected time taken in each CM iteration. On plugging in the value of $\epsilon_k$, both terms end up being of the same order and we arrive at the theorem.
%
%
The detailed proofs of the lemmas and the theorem can be found in Appendix A. 
\end{proof}

%% file: termination_conditions.tex
\section{Termination Rules}
\label{sec:termination_conditions}

In this section, we illustrate the construction of order-optimal termination rules for two representative and fundamentally different low-dimensional routines, one classical, one recent.  For simplicity, we focus on smooth objective functions. All notations are for a specific coordinate with coordinate index omitted.
\subsection{SGD} 
\label{sub:sgd}
 For a given initial point $x_1$, SGD proceeds by generating the following sequence of query points 
 \begin{equation}
 x_{t + 1} = \proj_{\cX}(x_t - \eta_t G(x_t, \xi_t)),    
 \end{equation}
 where $G(x_t, \xi_t)$ is the random gradient observed at $x_t$, $\{\eta_t\}_{t\ge 1}$ is the sequence of step sizes, and $\proj_{\cX}$ denotes the projection operator onto the convex set $\cX$ (restricted to the chosen coordinate with other coordinates fixed). 
The following lemma establishes the efficiency of the SGD routine with properly chosen hyperparameters and the order optimality of the termination rule. Based on Theorem~\ref{thm: PCM_p_consistent}, we can then conclude that  
the resulting PCM algorithm with $\{ \epsilon_k\} = \epsilon_0 \gamma^k$  is an efficient algorithm with a regret of $O(\log T)$.
\begin{lemma}
Consider the low-dimensional routine of SGD with step sizes given by $\eta_t = \mu/(1 + \nu t)$ with $\mu = \dfrac{\mu_0 \alpha}{2 g_{\max}^2}$ and $\nu = \dfrac{\mu_0 \alpha^2}{4 g_{\max}^2}$, where $g_{\max}$ is an upper bound on the second moment of the random gradient, $G(x, \xi)$, for all $x \in \cX$ and $\mu_0$ a properly chosen hyperparameter. Then SGD with the above chosen parameters is an efficient algorithm as defined in~\eqref{eq:eff_algo_def}. The termination rule given by $\tau(\epsilon) = \left \lceil \dfrac{2\beta g_{\max}^2}{\alpha^2 \epsilon} \right \rceil$ is order optimal as defined in Definition~\ref{def:tau_optimal}.  \label{lemma_SGD}
\end{lemma}

\begin{proof}
We give here a sketch of the proof. Details can be found in Appendix B. 
The order-optimality of the termination rule follows immediately from definition~\ref{def:tau_optimal}. For implementation in PCM, the constant $\mu_0$ is set to $\mu_0 \sim \Theta (\gamma^k)$ in iteration $k$ to ensure the adaptivity to the initial point. Using smoothness of $f$, we obtain $\displaystyle \E[f(x_t) - f(x^*)] \leq \beta \E[ |x_t - x^*|^2 ] \leq  \mu_0/(1 + \nu t)$, implying that SGD is a consistent algorithm with $p=1$. The choice of $\mu_0$ makes it an efficient algorithm with $\lambda = 1$.  
\end{proof}

\subsection{RWT} 
\label{sub:rwgd}

RWT (Random Walk on a Tree) proposed in \cite{Vakili2019a} is restricted to one-dimensional problems. There does not appear to be simple extension of RWT to high-dimensional problems. We show here that PCM offers a possibility and preserves its efficiency. \\

Without loss of generality, assume that the one-dimensional domain is the closed interval $[0,1]$. The basic idea of RWT is to construct an infinite-depth binary tree based on successive partitions of the interval. Each node of the tree represents a sub-interval with nodes at the same level giving an equal-length partition of $[0,1]$. The query point at each time is then generated based on a biased random walk on the interval tree that initiates at the root and is biased toward the node containing the minimizer $x^*$ (equivalently, the node/interval that sees a sign change in the gradient). When the random walk reaches a node, the two end points along with the middle point of the corresponding interval are queried in serial to determine, with a required confidence level $\bp$, the sign of $g(x)$ at those points. The test on the sign of $g(x)$ at any given $x$ is done through a confidence-bound based local sequential test using random gradient observations. The outcomes of the sign tests at the three points of the interval determines the next move of the random walk: to the child that contains a sign change or back to the parent under inconsistent test outcomes. For one-dimensional problems, the biased walk on the tree continues until $T$. \\

To extend RWT to high-dimensional setups within the PCM framework, we propose the following termination rule. Specifically, we leverage the structure of the local confidence-bound based sequential test in RWT. Note that the precision condition required at termination can be translated to an upper bound on the magnitude of the gradient. Since the local test is designed to estimate the sign of the gradient, it naturally requires more samples as the gradient gets closer to zero (i.e., the signal strength reduces while the noise persists). We hence impose the following termination rule: the current CM iteration terminates  once a sequential test draws more than $\displaystyle N_0(\epsilon) =  \frac{40 \sigma_0^2}{\alpha \epsilon} \log \left( \frac{2}{\bp} \log \left( \frac{80 \sigma_0^2}{\alpha \bp \epsilon} \right)\right)$ samples, where $\sigma_0^2 \geq \max_{i} \sigma_i^2$. \\

This threshold is so designed that when the number of samples in the sequential test exceeds that value, the gradient at that point is sufficiently small with high probability, leading to the required precision. It is interesting to note that the termination rule for SGD given in Lemma~\ref{lemma_SGD} is an open-loop design with pre-fixed termination time, while the termination rule proposed for RWT is closed-loop and adapts to random observations.  \\

The following lemma gives the 
regret order of the high-dimensional extension of RWT within the PCM framework. The detailed proof of the lemma is given in Appendix C.

\begin{lemma}
    The PCM-RWT algorithm with the chosen termination rule has a regret order of $O(\log T (\log\log T)^2)$.
\end{lemma}

%% file: extensions_discussions.tex
\section{Discussions}

\paragraph{Parallel and Distributed Computing} 
One of the major advantages of CD/CM based methods is their amenability to parallel and distributed computing. PCM naturally inherits the amenability of CM-based methods to parallel and distributed computing. Advantages of CD/CM methods in parallel and distributed computing have been well studied in the literature \cite{Richtarik2016a, Richtarik2016b, Bradley2011, Peng2013, Liu2015, Marecek2015,  Richtarik2016c}. It has been shown that the convergence of coordinate-based methods with parallel updates is guaranteed only when the parallel updates are aggregated in such a way that the combined update leads to a decrease in the function value as compared to the previous iterate. Such a condition is possible to verify and enforce when the objective function is deterministic and known, but presents a challenge in stochastic optimization. \\

To achieve parallelization of PCM that maintains the $p$-consistency as in the serial implementation, we draw inspiration from the technique proposed in \cite{Ferris1994} that leverages the convexity of the objective function.  \\

Assume there are $m < d$ independent cores, connected in parallel to a main server. The current iterate is passed to all the cores, with each receiving a different coordinate index, chosen at random. After the one dimensional optimization completes at each core, the next query point is set to the average of points returned by all the cores. A decrease in the function value  at the averaged point compared to the initial point is guaranteed by the convexity of the function. It can be shown that with a parallel implementation of PCM,  the ``effective'' dimension of the problem is reduced from $d$ to $d/m$ (see details in Appendix D).

\paragraph{Heavy-Tailed Noise} 

We have thus far focused on sub-Gaussian noise in gradient estimates. Extensions to heavy-tailed noise are straightforward. Similar to the case with sub-Gaussian noise, PCM, with the same optimal precision control as specified in Theorem~\ref{thm: PCM_p_consistent}, preserves the $p$-consistency and regret order of the low-dimension routine $\upsilon$ under heavy-tailed noise. In particular, for heavy-tailed noise with a finite $b^{\text{th}}$  moment ($b \in (1,2)$), both SGD and RWT are $(2 - 2/b)$-consistent and offer an optimal regret order of $O(T^{2/b - 1})$ (up to poly-$\log T$ factors)~\cite{Zhang2019,Vakili2019b}. The optimal order is preserved by PCM with them as CM routines.


%% file: simulation.tex
\section{Empirical Results}
\label{sec:simulations}

In this section, we compare the regret performance of PCM employing SGD with that of SCD (based on its online version developed in \cite{Wang2014}). 
We consider the problem of binary classification on the MNIST dataset~\cite{LeCunn1998} as a one-vs-rest classification problem in an online setting. We use regularized hinge loss as the loss function. At each time $t$, $\xi_t = (Y_t, Z_t)$ is drawn uniformly at random from the dataset. Then using the current classifier $\x_t$, we incur a random loss 
\begin{equation}
F(\x_t, \xi_t) = \max \{ 0, 1 - Z_t \ip{\x_t, Y_t} \} + \frac{\alpha}{2} \| \x_t\|^2
\end{equation}
and observe the partial gradient given as
\begin{align}
    G_{i_t}(\x_t, \xi_t) = -Z_t (Y_t)_{i_t} \1\{ 1 - Z_t \ip{\x_t, Y_t} > 0 \} + \alpha (\x_t)_{i_t}
\end{align}
where $i_t$ denotes the index of the coordinate chosen at time $t$.
 Both algorithms are randomly initialised with a point in $[-0.5, 0.5]^d$ and run for $T = 1000d$, where $d = 785$ is the dimensionality of the problem. The regularization parameter is set to $\alpha = 1.2 \times 10^{-2}$. The regret plotted is averaged over $10$ Monte Carlo runs. The SCD algorithm is run with stepsize $\eta_t = 5/\lceil t/10000\rceil$. The PCM algorithm is implemented using a SGD as the local optimization routine with $\gamma = 0.99999$, $\epsilon_0 = 0.1$ and step size of $\eta = 0.2$. The step size in each iteration was reduced by a factor of $\gamma$. The termination rule was set to $\lceil 1/2\epsilon_k\rceil$. The parameters in both algorithms were optimized based on a grid search. The results obtained for various digits are shown in Figure~\ref{Fig:PCM_SCD}. 
\begin{figure}[!htb]
\begin{subfigure}
     \centering
    \includegraphics[scale = 0.3]{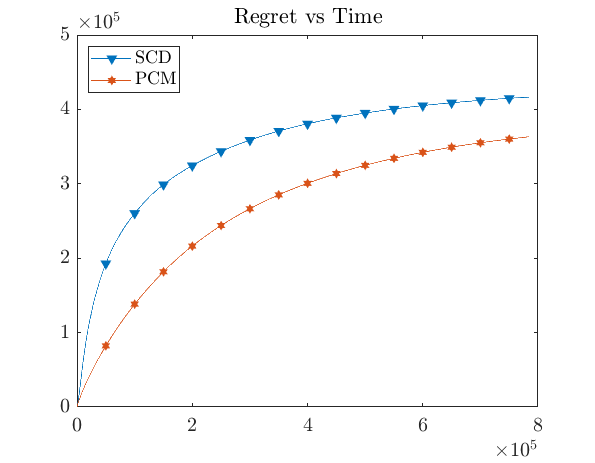}
\end{subfigure}
~
\begin{subfigure}
     \centering
    \includegraphics[scale = 0.3]{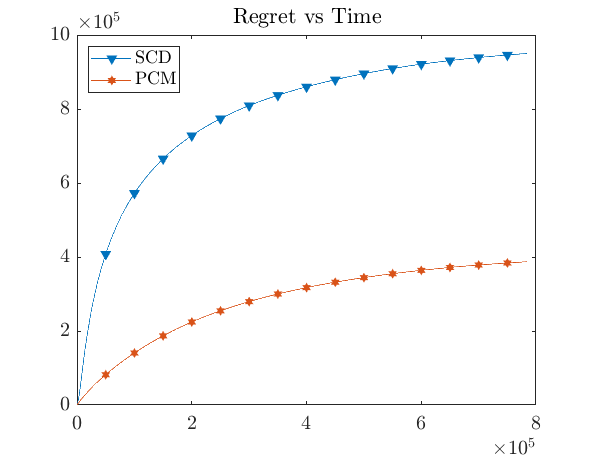}
\end{subfigure}
~
\begin{subfigure}
     \centering
    \includegraphics[scale = 0.3]{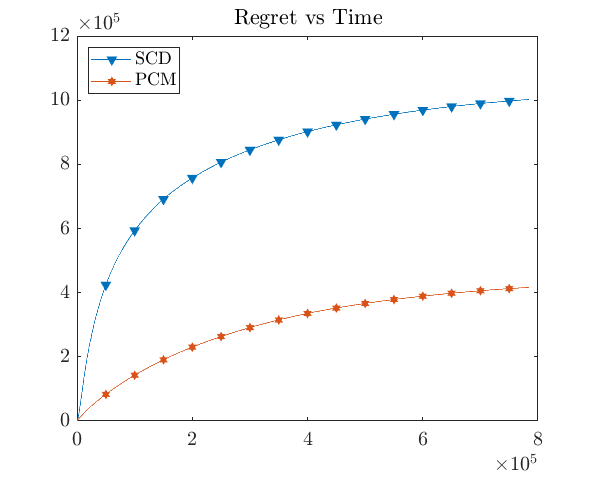}
\end{subfigure}
~
\begin{subfigure}
     \centering
    \includegraphics[scale = 0.3]{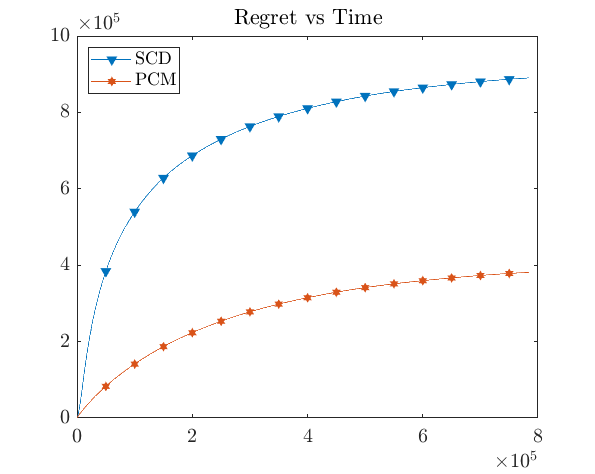}
\end{subfigure}
~
\begin{subfigure}
     \centering
    \includegraphics[scale = 0.3]{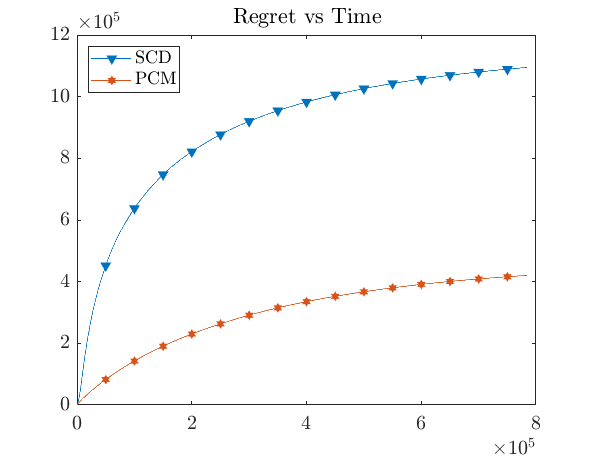}
\end{subfigure}
~
\begin{subfigure}
     \centering
    \includegraphics[scale = 0.3]{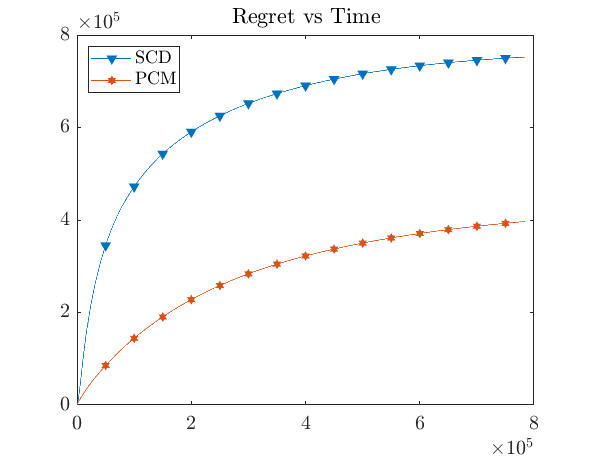}
\end{subfigure}
~
\begin{subfigure}
     \centering
    \includegraphics[scale = 0.3]{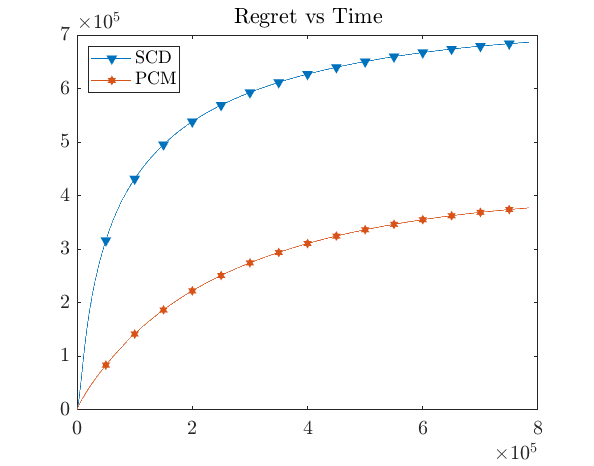}
\end{subfigure}
~
\begin{subfigure}
     \centering
    \includegraphics[scale = 0.3]{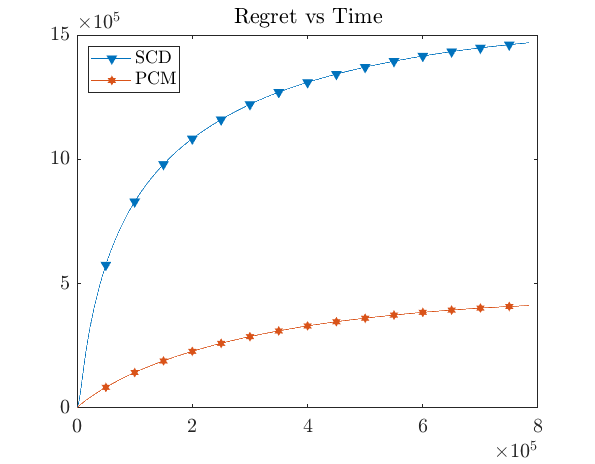}
\end{subfigure}
~
\begin{subfigure}
     \centering
    \includegraphics[scale = 0.3]{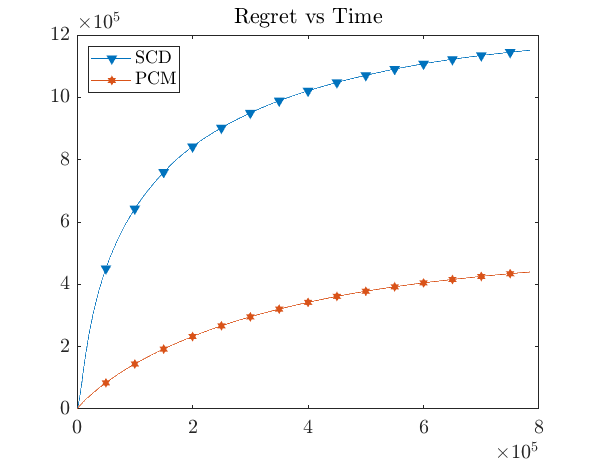}
\end{subfigure}
    \caption{Performance comparison of PCM against SCD for different digits in MNIST dataset. The digits are in ascending order when the matrix is traversed along rows. }
    \label{Fig:PCM_SCD}
\end{figure}
It is evident from the plots in Figure~\ref{Fig:PCM_SCD} that PCM has a significantly better performance. It suggests that the advantages of CM over CD type methods also hold in stochastic optimization.

%% file: Appendix.tex
\newpage

\section*{Appendix A: Proof for Theorem 1}

We first give a proof for Theorem~\ref{thm: PCM_p_consistent} using the two lemmas. Proofs for the lemmas then follow.  \\

We arrive at Theorem~\ref{thm: PCM_p_consistent} by bounding separately the two terms $R_1$ and $R_2$ in the regret decomposition in~\eqref{eq: regret_decomposition} for an arbitrary objective function $f\in\cF_{\alpha,\beta}$. Note that the consistency/efficiency level $p$ of an algorithm as defined in~\eqref{eq:p_consistency_def} and~\eqref{eq:eff_algo_def} is with respect to the worst-case objective function. This implies that when a $p$-consistent low-dimensional algorithm is employed for CM, the convergence rates along different coordinates may vary, depending on the reduction of $f$ to the specific coordinate. Let $p_k \geq p$ be the convergence rate in the coordinate $i_k$ chosen in the $k$-th iteration given  $\x_{-i_k}^{(k-1)}$. More specifically, the error with respect to the local minimum  $\x^*_{i_k,  \x^{(k-1)}}$ in the $i_k$-th coordinate decays as follows.  
\begin{equation}
\left(\E[f(x_{i_k,T},\x_{-i_k}^{(k-1)})] - f(\x^*_{i_k,  \x^{(k-1)}})\right) \sim \Theta(T^{-p_k}),
\end{equation}
where $x_{i_k,T}$ denotes the one-dimensional query point at time $T$. \\

We start by bounding $R_2$. 
\begin{align}
    R_2 &\leq \E \left[ \sum_{k = 1}^K \sum_{t = t_{k-1}+1}^{t_k} \left[ F(\x^*_{(i_k, \x^{(k-1)})}, \xi_t) - F(\x^*, \xi_t) \right]   \right] \nonumber \\
    & \leq \E \left[ \sum_{k = 1}^K \sum_{s = 1}^{\tau(\epsilon_k)} \left[ F(\x^*_{(i_k, \x^{(k-1)})}, \xi_s) - F(\x^*, \xi_s) \right]  \right] \nonumber \\
    & \leq \E \left[ \sum_{k = 1}^K \E \left[ \sum_{s = 1}^{\tau(\epsilon_k)} \left[ F(\x^*_{(i_k, \x^{(k-1)})}, \xi_s) - F(\x^*, \xi_s) \right]  \bigg| \x^{(k-1)} \right]  \right] \nonumber\\
    & \leq \E \left[ \sum_{k = 1}^K [f(\x^*_{(i_k, \x^{(k-1)}}) - f(\x^*)] \E[\tau(\epsilon_k)]  \right] \nonumber \\
    & \leq \E \left[ \sum_{k = 1}^K [f(\x^{(k-1)}) - f(\x^*)] \E[\tau(\epsilon_k)]  \right] \nonumber \\
    & \leq \E \left[ \sum_{k = 1}^K c_2 \gamma^k \epsilon_k^{-1/p_k}  \right] \nonumber\\
    & \leq \E \left[ \sum_{k = 1}^K  c_2' \left( \frac{1}{\gamma} \right)^{k \frac{1 - p_k}{p_k}} \right] \label{eq:R1_R2_common}
\end{align}
where the fourth line follows from Wald's Identity and $c_2, c_2' > 0$ are constants independent of $T$. To upper bound the expression obtained in~\eqref{eq:R1_R2_common}, we use that the total number of samples taken would be upper bounded by the length of the horizon.
\begin{align}
    \sum_{k = 1}^K   \E \left[ \tau(\epsilon_k) \right] \leq T.
\end{align}
Therefore, for some constant $c_3 > 0$ and independent of $T$, we have, 
\begin{align}
    \sum_{k = 1}^K   \left( \frac{1}{\gamma} \right)^{ \frac{k}{p_k}} \leq c_3 T.
\end{align}
Now using Jensen's inequality, we can write,
\begin{align}
    \frac{1}{K} \sum_{k = 1}^K   \left( \frac{1}{\gamma} \right)^{k \frac{1-p}{p_k}} & \leq \left( \frac{1}{K}  \sum_{k = 1}^K   \left( \frac{1}{\gamma} \right)^{ \frac{k}{p_k}}\right)^{1-p} \nonumber, \\
    \implies \sum_{k = 1}^K   \left( \frac{1}{\gamma} \right)^{k \frac{1-p}{p_k}} & \leq c_3' T^{1-p} K^p.
\end{align}
for some constant $c_3' > 0$. Note that the expression here is similar to the one obtained in~\eqref{eq:R1_R2_common} and in fact can be used to upper bound $R_2$. 
\begin{align}
    R_2 & \leq   c_2' \E \left[ \sum_{k = 1}^K  \left( \frac{1}{\gamma} \right)^{k \frac{1 - p_k}{p_k}} \right]  \nonumber  \\
     & \leq   c_2' \E \left[ \sum_{k = 1}^K  \left( \frac{1}{\gamma} \right)^{k \frac{1 - p}{p_k}} \right]  \nonumber  \\
     & \leq   c_2'' \E \left[ T^{1-p} K^p \right]  \nonumber  \\
     & \leq   c_2''  T^{1-p} \E\left[  K \right]^{p}  \nonumber  
\end{align}
where $c_2'' > 0$ is a constant independent of $T$ and the second step is obtained by noting $p_k \geq p$. Using the result from Lemma~\ref{lemma_k_lb} and plugging it in the above equation, we can conclude that $R_2$ is $O(T^{1-p}\log^p T)$. Note that for $p = 1$ this boils down to $O(\log T)$ as required.  \\

We now consider $R_1$. 
\begin{align}
    R_{1} & = \E \left[ \sum_{k = 1}^K   \sum_{t = t_{k-1}+1}^{t_k} \left[ F(\x_{t}, \xi_t) - F(\x^*_{(i_k, \x^{(k-1)})},\xi_t) \right] \right] \nonumber\\
     & = \E \left[ \sum_{k = 1}^K   \sum_{s = 1}^{\tau(\epsilon_k)} \left[ F(\x_{s + t_{k-1}}, \xi_{s + t_{k-1}}) - F(\x^*_{(i_k, \x^{(k-1)})}, \xi_{s + t_{k-1}}) \right] \right] \nonumber \\
     & =  \E \left[ \sum_{k = 1}^K  \E \left[ \sum_{s = 1}^{\tau(\epsilon_k)} \left[ F(\x_{s + t_{k-1}}, \xi_{s + t_{k-1}}) - F(\x^*_{(i_k, \x^{(k-1)})}, \xi_{s + t_{k-1}})   \right] \bigg| \tau(\epsilon_k) \right] \right] \nonumber 
\end{align}
Next we upper bound the above term separately for $p$-consistent and efficient routines. For $p$-consistent ($p < 1$) routines, we have, 
\begin{align}
    R_{1} & \leq \E \left[ \sum_{k = 1}^K  \E \left[ \sum_{s = 1}^{\tau(\epsilon_k)} \frac{c_1}{s^{p_k}} \bigg| \tau(\epsilon_k) \right] \right] \nonumber \\
     & \leq \E \left[ \sum_{k = 1}^K  c_1 \E \left[ (\tau(\epsilon_k))^{1 - p_k} \right] \right] \nonumber \\
     & \leq \E \left[ \sum_{k = 1}^K  c_1 \E \left[ \tau(\epsilon_k) \right]^{1 - p_k} \right] \label{eq:Jensen1} \\
     & \leq \E \left[ \sum_{k = 1}^K  c_1' \epsilon_k^{\frac{p_k-1}{p_k}} \right] \nonumber \\
     & \leq \E \left[ \sum_{k = 1}^K  c_1'' \left( \frac{1}{\gamma} \right)^{k \frac{1 - p_k}{p_k}} \right] \label{eq:R2_in_R1} 
\end{align}
where $c_1, c_1', c_1'' > 0$ are all constants independent of $T$ and~\eqref{eq:Jensen1} follows from Jensen's inequality. Note that the term obtained in~\eqref{eq:R2_in_R1} is of the same order as the one obtained in~\eqref{eq:R1_R2_common}. Therefore, using the same analysis as in the case of $R_2$, we can conclude that $R_1$ is also $O(T^{1-p}\log^p T)$ for $p$-consistent ($p < 1$) routines. Now for efficient routines we have $p_k = 1$ for all $k$. Along with the efficiency in leveraging the favorable initial conditions, we have
\begin{align}
    R_{1} & \leq \E \left[ \sum_{k = 1}^K  \E \left[ \left(f(\x^{(k-1)} - f(\x^*_{(i_k, \x^{(k-1)})})\right)^{\lambda} \sum_{s = 1}^{\tau(\epsilon_k)} \frac{b_2}{s} \bigg| \tau(\epsilon_k) \right] \right] \nonumber \\
     & \leq \E \left[ \sum_{k = 1}^K   b_2' \gamma^{(k-1)\lambda} \E \left[\log(\tau(\epsilon_k)) \right] \right] \nonumber \\
     & \leq \E \left[ \sum_{k = 1}^K   b_2'' (\epsilon_0\gamma^{k})^{\lambda}  \log \left(\E [\tau(\epsilon_k)]  \right) \right] \label{eq:Jensen2} \\
     & \leq \E \left[ \sum_{k = 1}^K   b_2'' \epsilon_k^{\lambda}  \log \left( \frac{b_3}{\epsilon_k} \right) \right] \nonumber \\
     & \leq \E \left[ \sum_{k = 1}^K   b_2'' \left( \epsilon_k^{\lambda}  \log \left( \frac{1}{\epsilon_k} \right) + \log(b_3) \epsilon_k^{\lambda} \right) \right] \nonumber \\
     & \leq \E \left[ \sum_{k = 1}^K   b_2'' \left( \frac{1}{\lambda e} + \log(b_3) \epsilon_0^{\lambda} \right) \right] \label{eq:lambda_upper_bound} \\
     & \leq b_4 \E[K]
\end{align}
where $b_2, b_2', b_2'', b_3, b_4 > 0$ are constants independent of $T$ and~\eqref{eq:Jensen2} and~\eqref{eq:lambda_upper_bound} are respectively obtained by Jensen's inequality and the fact that $ - x^{\lambda} \log(x)$ is uniformly upper bounded by $(\lambda e)^{-1}$ for all $x > 0$ and for all $\lambda > 0$. Using the upper bound on $\E[K]$ given from Lemma~\ref{lemma_k_lb} leads to the $O(\log T)$ order of $R_1$ for efficient algorithms. Combining the above bounds on $R_1$ and $R_2$, we arrive at the theorem.

\subsection*{Proof of Lemma 1}
Note that the first $K - 1$ iterations are complete by the end of the horizon of length $T$. We thus have, for some constants $b_1, b_1' > 0$,
	\begin{align}
	    T & \geq \E\left[\sum_{k = 1}^{K - 1} \E[\tau(\epsilon_k)] \right]  \nonumber \\
	    & \geq  \E\left[\sum_{k = 1}^{K - 1} b_1 \epsilon_k^{-1/p_k} \right]  \nonumber \\
	     & \geq  \E\left[\sum_{k = 1}^{K - 1} b_1 \epsilon_k^{-1} \right]  \nonumber \\
	     & \geq  \E\left[\sum_{k = 1}^{K - 1} b_1' \gamma^{-k} \right]  \nonumber \\
	     & \geq  \E\left[ b_1' \frac{\gamma^{-K} - \gamma^{-1}}{\gamma^{-1} - 1} \right]   \\
	\end{align}
	Therefore, we have that $\displaystyle \E \left[ \left( \frac{1}{\gamma}\right)^K \right] \leq \frac{T (1 - \gamma^{-1})}{b_1'} + \gamma^{-1}$. Taking logarithms on both sides and then applying Jensen's inequality, we obtain $\displaystyle \E[K] \leq \log_{\gamma^{-1}} \left(  \frac{T (1 - \gamma^{-1})}{b_1'} + \gamma^{-1}\right)$ as required.

\subsection*{Proof of Lemma 2}

The main idea of the proof revolves around the use of proximal operators which is similar to the convergence analysis in~\cite{Karimi2016}.  Specifically, for $f = \psi + \phi$, define
\begin{equation}
    \mathcal{D}_{\phi}(\x, \rho) := -2\rho \min_{\y \in \cX} \left[ \langle\nabla \psi(\x), \y -\x \rangle + \frac{\rho}{2}\| \y - \x\|^2 + \phi(\y) - \phi(\x)\right].
\end{equation}
Let us assume that we take a step of length $z_{i_k}$ along a fixed coordinate $i_k$. Therefore, using smoothness of $\nabla \psi$ and separability of $\phi$, we can write,
\begin{equation}
    f( \x^{(k-1)} + z_{i_k}\e_{i_k}) \leq f(\x) + z_{i_k}  [\nabla \psi(\x)]_{i_k} + \frac{\beta}{2}{z_{i_k}^2} + \phi_{i_k}(x_{i_k} + z_{i_k}) - \phi_{i_k}(x_{i_k}).
\end{equation}
Let $z_{i_k}$ be such that 
\begin{equation}
    z_{i_k} = \argmin_{t} \left[  t [\nabla \psi(\x)]_{i_k} + \frac{\beta}{2}{t^2} + \phi_{i_k}(x_{i_k} + t) - \phi_{i_k}(x_{i_k}). \right]
\end{equation}
Using the precision guaranteed by the termination rule and conditioning on $\x^{(k-1)}$, we have
\begin{align}
    \E[F(\x^k)| \x^{(k-1)}] & \leq f(\x^*_{(i_k, \x^{(k-1)})}) + \epsilon_k \nonumber \\
    & \leq f( \x^{(k-1)} + z_{i_k}\e_{i_k})  + \epsilon_k
\end{align}
Taking expectation over the coordinate index $i_k$, which is uniformly distributed over the set $\{1,2,\dots, d\}$, we can write,
\begin{align}
    \E[F(\x^k)| \x^{(k-1)}] & \leq \E_{i_k}[f( \x^{(k-1)} + z_{i_k}\e_{i_k})] + \epsilon_k,   \nonumber \\
    & \leq \E_{i_k}\left[f(\x^{(k-1)}) + z_{i_k}  [\nabla \psi(\x^{(k-1)})]_{i_k} + \frac{\beta}{2}{z_{i_k}^2}{2} + \phi_{i_k}(x^{(k-1)}_{i_k} + z_{i_k}) - \phi_{i_k}(x^{(k-1)}_{i_k}) \right]  +\epsilon_k,  \nonumber  \\
    & \leq f(\x^{(k-1)}) + \frac{1}{d} \sum_{i = 1}^d \min_{t_i} \left[ t_i  [\nabla \psi(\x^{(k-1)})]_{i} + \frac{\beta}{2}{t_i^2}{2} + \phi_{i}(x^{(k-1)}_{i} + t_i) - \phi_{i}(x^{(k-1)}_{i}) \right] + \epsilon_k,  \nonumber \\
    & \leq f(\x^{(k-1)}) + \frac{1}{d} \min_{t_1, t_2, \dots, t_d} \left[ \sum_{i = 1}^d  t_i  [\nabla \psi(\x^{(k-1)})]_{i} + \frac{\beta}{2}{t_i^2}{2} + \phi_{i}(x^{(k-1)}_{i} + t_i) - \phi_{i}(x^{(k-1)}_{i}) \right] + \epsilon_k, \nonumber   \\
    & \leq f(\x^{(k-1)}) + \frac{1}{d} \min_{\y} \left[ \langle F_1(\x^{(k-1)}), \y - \x^{(k-1)} \rangle + \frac{\beta}{2}\| \y - \x^{(k-1)} \| + \phi(\y) - \phi(x^{(k-1)}) \right] + \epsilon_k, \nonumber \\
    & \leq f(\x^{(k-1)}) - \frac{1}{2d\beta}  \mathcal{D}_g(\x^{(k-1)}, \beta) + \epsilon_k,  \nonumber \\
    & \leq f(\x^{(k-1)}) - \frac{\alpha}{d\beta}  ( f(\x^{(k-1)}) - f(\x^*) ) + \epsilon_k  \label{eq:det_CD_step}
\end{align}
where the step uses the proximal PL inequality for strongly convex functions described in~\cite{Karimi2016}. \\

Taking expectation over $\x^{(k-1)}$, we obtain,
\begin{equation}
 \E[F(\x^k)] - f(\x^*)  \leq  \left( \E[F ( \x^{(k-1)})] - f(\x^*) \right) \left( 1 - \frac{\alpha}{d\beta} \right) +   \epsilon_0 \gamma^{k}
\end{equation}

Let $\phi_k = \E[F(\x^k)] - f(\x^*)$. We claim that $\phi_k \leq F_0 \gamma^{k}$ where $\displaystyle F_0 = \max \left\{ f(\x^{(0)}) - f(\x^*),  \frac{\epsilon_0}{ (1 - \gamma) }  \right\}$. This can be proved using induction. For the base case, we have $\phi_0 = f(\x^{(0)}) - f(\x^*) \leq F_0$ by definition. Assume it is true for $k - 1$, then we have,
\begin{align}
    \phi_k & \leq \left( 1 - \frac{\alpha}{d\beta} \right) \phi_{k-1} +   \epsilon_0 \gamma^{k} \nonumber \\
    & \leq \gamma^2 (F_0 \gamma^{k-1}) + \epsilon_0 \gamma^{k} \nonumber \\
    & \leq F_0 \gamma^{k+1} +  \epsilon_0 \gamma^{k }\nonumber \\
    & \leq  \gamma^k \left( F_0 \gamma + \epsilon_0 \right) \nonumber\\
    & \leq F_0 \gamma^k.
\end{align}
The last step follows from the choice of $F_0$. This completes the proof.

\newpage

\section*{Appendix B: Proof of Lemma~\ref{lemma_SGD}}

We first prove the efficiency of SGD followed by the order optimality of the termination rule.
%
%

\subsection*{Efficiency of the SGD Routine}

Consider a one-dimensional stochastic function $F(x)$ with stochastic gradient given by $G(x)$. Let $x^*$ be the minimizer of the function, i.e., $x^* = \argmin_{x \in \cX} f(x)$ where $f(x) = \E[F(x)]$ and $\cX$ is the domain of the function. 
The iterates generated by SGD with initial point $x_0$ satisfy the following relation,
\begin{align}
    \E \left[ \| x_{t+1} - x^* \|^2\right] & = \E \left[  \| \proj_{\cX}(x_t -  \eta_t G(x_t)  - x^*) \|^2 \right]  \nonumber\\
    & \leq \E \left[  \| x_t -  \eta_t G(x_t)  - x^* \|^2 \right] \nonumber \\
    & \leq \E \left[   \| x_t - x^* \|^2 - 2 \eta_t \ip{ G(x_t), x_t - x^*}  + \eta_t^2 \|G(x_t)\|^2 \right] \nonumber \\
    & \leq \E \left[   \| x_t - x^* \|^2  \right] - 2 \eta_t  \E\left[ \alpha \| x_t - x^* \|^2  \right]   + \eta_t^2 \E\left[ \|G(x_t)\|^2 \right] \nonumber \\
    & \leq (1 - 2\eta_t \alpha) \E \left[   \| x_t - x^* \|^2  \right]  + \eta_t^2 g_{\max}^2 \label{eq:SGD_general}
\end{align}
Next we show that the iterates satisfy $\displaystyle \E\left[ \| x_t - x^* \|^2 \right] \leq \frac{\mu_0}{1 +  \nu t} $ for all $t \geq 0$ based on an inductive argument. The base case is ensured by choosing $\mu_0$ satisfying $\mu_0 \geq \E[|x_0 - x^*|^2]$.  
For the induction step, note that the stepsizes are chosen as $\eta_t = \dfrac{\mu}{1 + \nu t}$ with $\mu = \dfrac{\mu_0 \alpha}{2 g_{\max}^2}$ and $\nu = \dfrac{\mu_0 \alpha^2}{4 g_{\max}^2}$. We continue with~\eqref{eq:SGD_general} as follows. 
\begin{align}
    \E \left[ \| x_{t+1} - x^* \|^2\right] & \leq (1 - 2\eta_t \alpha) \E \left[   \| x_t - x^* \|^2  \right]  + \eta_t^2 g_{\max}^2 \nonumber \\
    & \leq \left(1 - 2\frac{ \mu \alpha}{1 + \nu t} \right)  \frac{\mu_0}{1 +  \nu t}    + \frac{ \mu^2}{(1 +  \nu t)^2} g_{\max}^2 \nonumber \\
    & \leq \frac{\mu_0}{1 +  \nu (t+1)} + \left( \frac{\mu_0}{1 +  \nu t} - \frac{\mu_0}{1 +  \nu (t+1)} \right)  + \frac{ \mu}{(1 +  \nu t)^2} (\mu g_{\max}^2 - 2 \mu_0\alpha) \nonumber \\
    & \leq \frac{\mu_0}{1 +  \nu (t+1)}  + \frac{\mu_0^2 }{(1 +  \nu t)^2} (\mu^2 g_{\max}^2 - 2 \mu \mu_0 \alpha + \mu_0 \nu) \\
    & \leq \frac{\mu_0}{1 +  \nu (t+1)}  + \frac{\mu_0^2 }{(1 +  \nu t)^2} \left(\frac{\mu_0^2 \alpha^2}{4 g_{\max}^4} g_{\max}^2 -  \frac{ \mu_0^2 \alpha^2}{ g_{\max}^2} + \dfrac{\mu_0^2 \alpha^2}{4 g_{\max}^2} \right) \\
    & \leq \frac{\mu_0}{1 +  \nu (t+1)}  \label{eq:SGD_iterates}
\end{align}
Therefore, the iterates generated by SGD satisfy $\displaystyle \E\left[ \| x_t - x^* \|^2 \right] \leq \frac{\mu_0}{1 +  \nu t} $ for all $t \geq 0$. 

To ensure efficiency with respect to the initial point $x_0$, $\mu_0$ should be of the order $\mu_0 \leq C \E[|x_0 - x^*|^2]$ as $x_0$ goes to $x^*$ for some $C > 0$ (see below how this can be ensured within the PCM framework). Based on the strong convexity and smoothness of the function, the condition on the iterates can be translated to a condition on the function values as given below
\begin{align}
    \E[F(x_t) - f(x^*)] \leq \frac{\beta C}{\alpha} \frac{\E[f(x_0) - f(x^*)]}{1 +  \nu t},
\end{align}
which implies that SGD is an efficient policy with $\lambda = 1$. \\ 

For implementation in PCM, the choice of $\mu_0$ can be simplified using the relation on the CM iterates outlined in Lemma~\ref{lemma_CM}. In iteration $k$, $\x^{(k-1)}$ is the initial point, therefore, we can write, $\E\left[ \| \x^{(k-1)} - \x^*_{(i_k, x^{(k-1)})} \|^2 \right] \leq \dfrac{2}{\alpha} \E\left[ f(\x^{(k-1)}) - f(\x^*_{(i_k, x^{(k-1)})}) \right] \leq \E\left[ f(\x^{(k-1)}) - f(\x^*) \right] \leq F_0 \gamma^{k-1}$. Thus, for an appropriate choice of $\mu_0$ for the first iteration, its value for consequent iterations can be obtained by the relation $\mu_0(k) = \gamma\mu_0(k-1)$, where $\mu_0(k)$ is the value of $\mu_0$ used in iteration $k$.

\subsection*{Order Optimality of the Termination Rule}

The correctness of the termination rule follows in a straightforward manner from the relation obtained on the iterates in the previous part. Using smoothness of the function and the relation obtained in~\eqref{eq:SGD_iterates}, we have, 
 $\displaystyle \E[F(x_t) - f(x^*)] \leq \frac{\mu_0 \beta}{2(1 + \nu t)} $. Let $t_0$ be such that, $ \dfrac{\mu_0 \beta}{2(1 + \nu t_0)} \leq \epsilon$. On rearranging this equation, we obtain $\displaystyle t_0 \geq \frac{\mu_0 \beta}{2 \epsilon \nu} - \frac{1}{\nu}$. Therefore, for all $t \geq t_0$, we have $\displaystyle \E[F(x_t) - f(x^*)] \leq \epsilon$. Since our choice of termination rule satisfies the above condition, we can conclude that our termination rule ensures the required precision. The order optimality of the termination also follows directly from the expression.


\newpage

\section*{Appendix C}

In this section, we analyze the performance of the Random Walk on a Tree (RWT) under the PCM setup. We begin with briefly outlining the RWT algorithm for PCM setup followed by the termination rule and then conclude the section with the performance analysis of PCM-RWT. \\

Let $F(x, \xi)$ be the one dimensional stochastic function to be minimized and $G(x, \xi)$ denote its stochastic gradient while $f(x)$ and $g(x)$ respectively denote their expected values. Also we assume that $|g(x)| \leq g_{\max}$ for all $x \in \cX'$, where $\cX'$ is the domain of the function.


\subsection*{RWT Algorithm for PCM}

In the $k^{\text{th}}$ iteration of PCM, optimization is carried out in the along the direction $i_k$, chosen in that iteration. Therefore, the one dimensional domain is the interval given by $\{ x : (x, \x_{-i_k}^{(k-1)}) \in \cX \}$, that is, all the points in the domain whose all but the $i_k^{\text{th}}$ coordinates are same as that of $\x^{(k-1)}$. The length of this interval depends upon the diameter of the domain along the $i_k^{\text{th}}$ direction. Without loss of generality, we assume that the one-dimensional domain is the closed interval $[0,1]$ (as the extension to any interval $[a,b]$ is straightforward). \\

The basic idea of RWT is to construct an infinite-depth binary tree based on successive partitions of the interval. Each node of the tree represents a sub-interval with nodes at the same level giving an equal-length partition of $[0,1]$. The query point at each time is then generated based on a biased random walk on the interval tree that initiates at the root and is biased toward the node containing the minimizer $x^*$ (equivalently, the node/interval that sees a sign change in the gradient). When the random walk reaches a node, the two end points along with the middle point of the corresponding interval are queried in serial to determine, with a required confidence level $\bp$, the sign of $g(x)$ at those points. The test on the sign of $g(x)$ at any given $x$ is done through a confidence-bound based local sequential test using random gradient observations. The outcomes of the sign tests at the three points of the interval determines the next move of the random walk: to the child that contains a sign change or back to the parent under inconsistent test outcomes. \\

A crucial aspect of the above algorithm is the local sequential test. Let the sample mean of $s$ samples of the stochastic gradient at a point $x \in \cX'$ be denoted as $\displaystyle \bar{G}_{s}(x) = \frac{1}{s} \sum_{t = 1}^{s} G(x, \xi_t)$. The sequential test in RWT for sub-Gaussian noise is given below. For heavy-tailed noise, the only required change to RWT is in the confidence bounds used in the sequential test (see~\cite{Vakili2019b}). 

\begin{figure}[H]
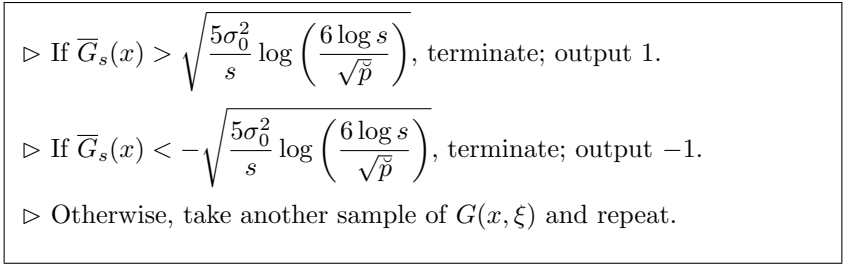

\begin{center}
\noindent\fbox{
\parbox{4.25in}{
{
$\rhd$ If $\displaystyle \overline{G}_s(x) > \sqrt{\frac{5 \sigma_0^2}{ s} \log\left(\frac{6 \log s}{\sqrt{ \bp}}\right)}$, terminate; output $1$.\\[0.5em]
$\rhd$ If $\displaystyle \overline{G}_s(x) < -\sqrt{\frac{5 \sigma_0^2}{ s} \log\left(\frac{6\log s}{\sqrt{ \bp}}\right)}$, terminate; output $-1$.\\[0.5em]
$\rhd$ Otherwise, take another sample of $G(x, \xi)$ and repeat.\\
}
}}
\caption{The sequential test at a sampling point $x$ under sub-Gaussian noise.}\label{Fig:STest-subG}
\end{center}
\end{figure}
where $\bp$ is the confidence parameter for the sequential test. To ensure the bias in the random walk, $\bp$ is set to a value in $(0, 1 - 2^{-1/3})$. \\

The RWT algorithm as described in~\cite{Vakili2019a, Vakili2019b} initializes the random walk at the root of the tree as there is no prior information about the location of the minimizer. However, if we have some prior information about the location of the minimizer, we can initialize the random walk at a lower level in the tree. This enables us to give higher preference to the region where the minimizer is likely to be located, thereby reducing the expected time to convergence. Consequently, such an initialization allows RWT to leverage favorable initial conditions. Therefore, for PCM, we initialize RWT at the node which contains the initial point and is at a level where the interval length is lesser than $\displaystyle \sqrt{\log_2 \left( \frac{\beta \sqrt{2}}{\sqrt{\alpha \epsilon}} \right) \frac{2 \mu_0}{\alpha}}$, where $\mu_0$ is a carefully chosen hyperparameter and $\epsilon$ is the required precision.  If the threshold exceeds $1$, then we begin at the root. The significance of this choice of initialization and the allowed values of $\mu_0$ are discussed in a later section which outlines an upper bound on $\E[\tau(\epsilon)]$.


\subsection*{Termination Rule}

We begin with a lemma that states the correctness of the termination rule and also relate the expected second moment of the gradient of the final point to the required precision $\epsilon$. Recall that the termination rule specified that if at a certain point the number of samples taken in a sequential test exceeds $N_0(\epsilon) = \frac{40 \sigma_0^2}{\alpha \epsilon} \log \left( \frac{2}{\bp} \log \left( \frac{80 \sigma_0^2}{\alpha \bp \epsilon} \right)\right)$ the algorithm must terminate, returning the current point being probed.

\begin{lemma}
Let $x_{\tau(\epsilon)}$ denote the final point obtained under the termination rule. Then we have the following relations
$\E[f(x_{\tau(\epsilon)}) - f(x^*)] \leq \epsilon$ and $\E[g^2(x_{\tau(\epsilon)})] \leq 2 \alpha \epsilon$.
\end{lemma}
\begin{proof}
The first part of the lemma directly follows from the second part using the strong convexity of the function. Since $f$ is strongly convex, we have
$\displaystyle \E[f(x_{\tau(\epsilon)}) - f(x^*)] \leq \frac{1}{2 \alpha}\E[g^2(x_{\tau(\epsilon)})] \leq \epsilon$ as required. Hence, we just focus on the proving the bound on the gradient. \\ 

To obtain the bound on the gradient, we leverage the primary idea underlying the design of the threshold in the termination rule. The threshold is designed to ensure that the gradient at the point at which the algorithm terminates is sufficiently small with high probability. We use this high probability bound to obtain the required bound on the second moment of the gradient. \\

Define $\rho := \dfrac{\alpha \epsilon}{2}$. We claim that under the given termination rule, $|g(x_{\tau(\epsilon)}| \leq \rho$ holds with high probability. To prove the claim, we consider the probability that the random number of samples taken in a sequential test, denoted by $\hat{T}$, exceed any number $n$. For any point with $g(x) > 0$, we have, 
\begin{align}
    \mathbb{P}[\hat{T}> n] & \leq \mathbb{P} \left[\forall s \leq n: \overline{G}_s(x) + \sqrt{\frac{5 \sigma_0^2}{ s} \log\left(\frac{6 \log s}{\sqrt{ \bp}}\right)} >0,~\text{and}~\overline{G}_{s}-\sqrt{\frac{5 \sigma_0^2}{ s} \log\left(\frac{6 \log s}{\sqrt{ \bp}}\right)} <0  \right], \nonumber \\ 
    & \leq \mathbb{P} \left[\forall s \leq n: \overline{G}_{s}-\sqrt{\frac{5 \sigma_0^2}{ s} \log\left(\frac{6 \log s}{\sqrt{ \bp}}\right)} <0  \right], \nonumber \\ 
     & \leq \mathbb{P} \left[ \overline{G}_{n}-\sqrt{\frac{5 \sigma_0^2}{ n} \log\left(\frac{6 \log n}{\sqrt{ \bp}}\right)}<0  \right],  \nonumber \\
     & \leq \mathbb{P} \left[ \overline{G}_{n}- \E_{\xi}[G(x, \xi)]  < \sqrt{\frac{5 \sigma_0^2}{ n} \log\left(\frac{6 \log n}{\sqrt{ \bp}}\right)} - g(x)] \right], \nonumber  \\
     &\leq \exp \left( -\frac{n}{2 \sigma_0^2} \left( \sqrt{\frac{5 \sigma_0^2}{ n} \log\left(\frac{6 \log n}{\sqrt{ \bp}}\right)} - g(x) \right)^2 \right). \label{eq:tail_prob_rwt}
\end{align}
The threshold $N_0(\epsilon)$ can be equivalently written in terms of $\rho$ as $s_0(\rho) = \dfrac{20 \sigma_0^2}{\rho^2} \log \left( \dfrac{2}{\bp} \log \left( \dfrac{40\sigma_0^2}{\bp \rho^2} \right)\right)$. For all $n > s_0(\rho)$, we have, 
\begin{align}
    \mathbb{P}[\hat{T}> n] & \leq \exp \left( -\frac{n}{2 \sigma_0^2} \left( \frac{\rho}{2} - g(x) \right)^2 \right) \label{eq:tail_prob_thresh}
\end{align}
This can be obtained by plugging $n = s_0(\rho)$ in the upper bound in~\eqref{eq:tail_prob_rwt}. A more detailed analysis of this step can be found in Appendix B in~\cite{Vakili2019a}. A similar analysis can be carried out for any point with $g(x) < 0$. Using~\eqref{eq:tail_prob_thresh}, we can conclude that if the number of samples in a local test at a point $x$ exceed $s_0(\rho)$, then $|g(x)| \leq \rho$ with probability at least $1 - \delta_0$ where $\displaystyle \delta_0 = \exp \left(-\frac{s_0(\rho) \rho^2}{8 \sigma_0^2} \right)$. 

We can now bound the second moment of the gradient as follows by noting that $\delta_0 \leq 1/2$
\begin{align}
    \E[g^2(x_{\tau(\epsilon)})] & \leq \rho^2 \mathbb{P}[g(x_{\tau(\epsilon)}) \leq \rho] + \E[g^2(x_{\tau(\epsilon)}) \1_{\{g(x_{\tau(\epsilon)}) > \rho\}}], \nonumber \\
    & \leq \rho^2 + \sum_{r = 1}^{\infty} (r + 1)^2 \rho^2 \mathbb{P}( r \rho < g(x_{\tau(\epsilon)}) \leq (r + 1) \rho), \nonumber \\
    & \leq \rho^2 + \sum_{r = 1}^{\infty} (r + 1)^2 \rho^2 \mathbb{P}( g(x_{\tau(\epsilon)}) > r \rho), \nonumber \\
    & \leq \rho^2 + \sum_{r = 1}^{\infty} (r + 1)^2 \rho^2 \exp \left(-\frac{s_0(\rho)}{2 \sigma_0^2} \left( \frac{\rho}{2} - r \rho \right)^2\right), \nonumber \\
    & \leq \rho^2 + \sum_{r = 1}^{\infty} (r + 1)^2 \rho^2 \exp \left(-\frac{s_0(\rho) \rho^2}{8 \sigma_0^2} (2r - 1)^2\right), \nonumber \\
    & \leq \rho^2 + \sum_{r = 1}^{\infty} (r + 1)^2 \rho^2 \delta_0^{(2r - 1)^2},  \nonumber\\
    & \leq \rho^2 + 2.01\rho^2,  \nonumber \\
    & \leq 4 \rho^2.
\end{align}
By plugging in $\rho = \sqrt{\dfrac{\alpha \epsilon}{2}}$, we arrive at the required result.
\end{proof}

As it is easier to analyze expressions in terms of the gradient and not the function values, we will use the expressions in terms of $\rho$ for the rest of the section keeping in mind its relation with the required precision of $\epsilon$.

\subsection*{Upper Bound on $\E[\tau(\epsilon)]$}

To bound the expected number of samples taken in one iteration of the PCM-RWT algorithm with precision $\epsilon$, $\E[(\tau(\epsilon)]$, we need to obtain a bound on the number of steps taken by the random walk before termination. The bound on $\E[\tau(\epsilon)]$ follows by noting that the number of samples taken in each sequential test before termination is bounded by the threshold specified in the termination rule. For the bound on the number of steps in the random walk, we note that as the random walk gets to a deeper level in the tree, the magnitude of the gradient reduces. Consequently, the probability that the number of samples taken in the sequential test will cross the threshold increases as the walk goes to a deeper level in the tree. These decreasing tail probabilities can then be used to obtain a bound on the expected number of steps in the random walk. \\

Assume the minima to be $x^* = 0$. Such an assumption leads to no loss of generality as the analysis can easily be modified for any point in the interval and for any interval of any length. We begin the analysis for the case when the random walk is initialized at the root node. This analysis can be easily modified to accommodate the initialization at a deeper level. \\

We divide the tree into a sequence of subtrees given by $\cT_1, \cT_2, \dots$ where for all $i =1, 2, \dots $, the subtree $\cT_i$ contains the node corresponding to the interval $[0, 2^{-(i-1)}]$ and its right child along with all its children. Thus, $\cT_i$'s are half trees rooted at level $i - 1$, along with their root. This construction is similar to the one outlined in~\cite{Wang2018}. Since the random walk is biased towards the minimizer, therefore given the construction of $\cT_i$, the probability that random walk is still in one of such subtrees would decrease with time. To formalize this idea, we consider the last passage times of any subtree $\cT_{i}$. Let $\tau_1$ denote the last passage time to $\cT_1$. \\

The analysis of the last passage time of $\cT_1$ can be mapped to the problem of a random walk on the set $S = \{ -1, 0, 1, 2, \dots \}$.
The underlying idea is that each non-negative integer can be mapped to the corresponding level in subtree. Our random walk on the tree can between different levels is then equivalent to a random walk on these integers. The equivalence follows by noting that the specific intervals on any level are all identical as they do not contain the minimizer and thus can be abstracted into a single entity. Hence, we map the root node to $0$, and set of nodes at level $j$ in subtree $\cT_1$ to integer $j$ for $j > 0$. Lastly, we map the left subtree containing all nodes in the interval $[0, 0.5]$ to $-1$, which corresponds to an exit from the subtree $\cT_1$. \\

The random walk can be modelled as a Markov chain on the set $S$, where $\mathbb{P}(j \to j+1) = 1 - p$ for all $j \in S$, $\mathbb{P}(j \to j-1) = p$ for all $j \geq 0$ and $\mathbb{P}(-1 \to -1) = p$. The probability $p = \bp^3 > 0.5$ is the probability of moving in correct direction where $\bp$ is the confidence level in the sequential test. The initial state is $0$. \\

Since $-1$ denotes the state corresponding to exiting the subtree $\cT_1$, therefore our random walk still being in $\cT_1$ after $n$ steps is the same  as the Markov Chain being in a state $j$ for $j \geq 0$ after $n$ steps. Furthermore, since the Markov Chain was initialized at $0$, therefore being in state $j \geq 0$ implies that the number of steps taken in the positive direction are at least as many as those taken in the negative direction. Combining all these ideas along with noting the specific structure of the transition matrix, we can conclude that 
\begin{align}
    \mathbb{P}(\tau_1 > n) = \mathbb{P}(Z \leq n/2),
\end{align}
where $Z \sim \mathrm{Bin}(n,p)$. Writing expectation as the sum of tail probabilities, 
\begin{align}
    \E[\tau_1] & = \sum_{n = 0}^{\infty} \mathbb{P}(\tau_1 > n), \nonumber\\
    & = \sum_{n = 0}^{\infty} \mathbb{P}(Z \leq n/2), \nonumber\\
    & = \sum_{n = 0}^{\infty} \exp(-2(p -1/2)^2 n),\nonumber \\
    & = \frac{1}{1 - \exp(-2(p - 1/2)^2)}.
\end{align}
The third step is obtained using Hoeffding's inequality. We can leverage the symmetry of the random walk and the binary tree to obtain the expected last passage time for any other subtree $\cT_i$. \\

Let for all $i \geq 1$, $N_{\cT_i}$ denote the random number of steps taken in subtree $\cT_i$ before exiting that subtree and $E_i$ denote the event that the random walk does not terminate in tree $\cT_i$. If $N_{RW}$ denotes the random number of steps taken by the random walk before termination then
\begin{align}
    \E[N_{RW}] & = \E[N_{\cT_1}] + \sum_{i = 2}^{\infty} \mathbb{P}\left( \bigcap_{j =1}^{i-1} E_j \right) \E \left[ N_{\cT_i}\bigg|\bigcap_{j =1}^{i-1} E_j \right]. \label{eq:N_RW}
\end{align}
By definition we have $\E[N_{\cT_1}] = \E[\tau_1]$. Furthermore, one can note that due to symmetry in the structure of the binary tree, $\displaystyle  \E \left[ N_{\cT_i}\bigg|\bigcap_{j =1}^{i-1} E_j \right] = \E[\tau_1]$. Hence, to evaluate~\eqref{eq:N_RW} we need to find a bound on $\displaystyle \mathbb{P}\left( \bigcap_{j =1}^{i-1} E_j \right)$, the probability that the random walk does not terminate in $\cT_j$ for $j = 1,2, \dots, i-1$ and $i \geq 2$. To bound this probability, consider the event that local sequential test takes less than $s_0(\rho)$ samples before termination when the magnitude of the gradient of the point being sampled is less than $\rho$. Let the event be denoted by $E_{f}(\rho)$ and let $\mathbb{P}(E_f(\rho)) \leq \eta_{\rho}$ for some $\eta_{\rho} < 1$. \\

Note that for any level $i > \log_2(\beta/\rho)$,the length of the interval at this level would be lesser than $\rho/\beta$. Using the smoothness of the function, it follows that the magnitude of gradient of any point probed in $\cT_i$ for $i > \log_2(\beta/\rho)$ would be lesser than $\rho$. Therefore, for every $i > \log_2(\beta/\rho)$, if $E_i$ occurs then $E_f(\rho)$ would definitely have occurred. Consequently, for all $i > i_0$, 
\begin{align}
    \mathbb{P}\left( \bigcap_{j =1}^{i-1} E_j \right) \leq \eta_{\rho}^{i - i_0},
\end{align}
where $i_0 = \lceil\log_2(\beta/\rho)\rceil$. For $i \leq i_0$, we can crudely upper bound this probability with $1$. Plugging these relations into~\eqref{eq:N_RW}, we obtain,
\begin{align}
    \E[N_{RW}] & = \E[N_{\cT_1}] + \sum_{i = 2}^{\infty} \mathbb{P}\left( \bigcap_{j =1}^{i-1} E_j \right) \E \left[ N_{\cT_i}\bigg|\bigcap_{j =1}^{i-1} E_j \right],  \nonumber \\
    & \leq  \E[\tau_1]\left( i_0  + 1 + \sum_{i = i_0 + 1}^{\infty} \eta_{\rho}^{i - i_0} \right), \nonumber \\
     & \leq  \frac{1}{1 - \exp(-2(p - 1/2)^2)}\left(\lceil\log_2(\beta/\rho)\rceil  + 1 +   \sum_{i = 1}^{\infty} \eta_{\rho}^{i} \right), \nonumber\\
      & \leq  \frac{1}{1 - \exp(-2(p - 1/2)^2)}\left(\log_2(\beta/\rho) + 2+   \frac{\eta_{\rho}}{1 - \eta_{\rho}} \right). \label{eq:ex_samp_RWT_ub}
\end{align}

The above analysis provides an upper bound for the number of steps taken by the random walk when it it initialized at the root node. However, as mentioned previously, we would want the RWT to be initialized at a deeper level in the tree to leverage the favorable initial conditions. We can perform a similar analysis for number of steps taken by random walk in the case when RWT is initialized at a deeper level to leverage the favorable initial conditions. \\

As given in the description of PCM-RWT, we initialize the algorithm the node which contains the initial point and at a level where the interval length is lesser than $\displaystyle \sqrt{\log_2 \left( \frac{\beta}{\rho} \right) \frac{2 \mu_0}{\alpha}}$, where $\mu_0$ is a carefully chosen hyperparameter. If the threshold exceeds $1$, then we begin at the root. To analyze the number of steps taken by the random walk, we consider the event that $\displaystyle |x_0 - x^*|^2 \leq \log_2 \left( \frac{\beta}{\rho} \right) \frac{2 \mu_0}{\alpha}$, where $x_0$ is randomly chosen. We denote the event by $E_{x_0}$. Under this event, we can carry out a similar analysis as before, with a minor change that instead of $i_0$, the maximum depth would be $\displaystyle i_1 \leq \log_2 \left( \sqrt{\log_2 \left( \frac{\beta}{\rho} \right) \frac{2 \mu_0}{\alpha}} \frac{\beta}{\rho}\right) + 1$. Under the case the above event does not occur, the random walk would have to take no more than an additional $\displaystyle i_2 \leq \log_2 \left( \sqrt{\log_2 \left( \frac{\beta}{\rho} \right) \frac{2 \mu_0}{\alpha}} \right) + 1$ steps before the previous analysis is again applicable. If $N_{RW-\text{new}}$ denotes the random number of steps taken by the random walk under this initialization scheme, then on combining the above results, we can write,
\begin{align}
    \E[N_{RW-\text{new}}] & \leq \mathbb{P}(E_{x_0}) \left( \zeta_p \left( \log_2 \left( \sqrt{\log_2 \left( \frac{\beta}{\rho} \right) \frac{2 \mu_0}{\alpha}} \frac{\beta}{\rho}\right) + \hat{\eta}_{\rho}\right) \right), \nonumber \\
    & \ \ \ \ \ \ \ \ \  + \mathbb{P}(E_{x_0}^c) \left( \zeta_p \left( \log_2 \left( \sqrt{\log_2 \left( \frac{\beta}{\rho} \right) \frac{2 \mu_0}{\alpha}}\right) + \log_2 \left( \frac{\beta}{\rho}\right) + 1 +  \hat{\eta}_{\rho}\right) \right), \label{eq:N_RW_new}
\end{align}
where $E^c$ denotes the complement of an event $E$, $\zeta_p = \dfrac{1}{1 - \exp(-2(p - 1/2)^2)}$ and $\hat{\eta}_{\rho} =  2+   \dfrac{\eta_{\rho}}{1 - \eta_{\rho}}$. We can bound $\mathbb{P}(E_{x_0}^c)$ using Markov's inequality as follows,
\begin{align}
    \Pr \left( |x_0 - x^*|^2 > \log_2 \left( \frac{\beta}{\rho} \right) \frac{2 \mu_0}{\alpha} \right) & \leq \Pr \left( f(x_0) - f(x^*) > \log_2 \left( \frac{\beta}{\rho} \right) \mu_0 \right), \nonumber \\
     & \leq {\E \left[  f(x_0) - f(x^*) \right]} \left( \log_2 \left( \frac{\beta}{\rho} \right) \mu_0 \right)^{-1}. \label{eq:Markov_init}
\end{align}
Setting $\mu_0 = \E \left[  f(x_0) - f(x^*) \right]$ and plugging~\eqref{eq:Markov_init} in~\eqref{eq:N_RW_new}, we obtain,
\begin{align}
    \E[N_{RW-\text{new}}] & \leq  \left( \zeta_p \left( \frac{1}{2}\log_2 \left( \log_2 \left( \frac{\beta}{\rho} \right) \frac{2 \E \left[  f(x_0) - f(x^*) \right]}{\alpha} \frac{\beta^2}{\rho^2}\right) + \hat{\eta}_{\rho}\right) \right) \nonumber \\
    &  +  \left( \log_2 \left( \frac{\beta}{\rho} \right)  \right)^{-1}\left( \zeta_p \left( \frac{1}{2}\log_2 \left( \log_2 \left( \frac{\beta}{\rho} \right) \frac{2 \E \left[  f(x_0) - f(x^*) \right]}{\alpha} \right) +  \log_2 \left( \frac{\beta}{\rho}\right) + 1 +  \hat{\eta}_{\rho}\right) \right), \nonumber \\
    & \leq  \left( \zeta_p \left( \frac{1}{2}\log_2 \left( \log_2 \left( \frac{\beta}{\rho} \right) \frac{2 \E \left[  f(x_0) - f(x^*) \right]}{\alpha} \frac{\beta^2}{\rho^2}\right) + \hat{\eta}_{\rho}\right) \right) \nonumber \\
    &  \ \ \ \ \ \ \ \ \ \ \ +  \zeta_p \left( \frac{ \log_2(\beta/g_{\max}) + 0.5\log_2(2g_{\max}/\alpha) + \hat{\eta}_{\rho} + 1.5}{\log_2(\beta/g_{\max})}  \right).  \label{eq:N_RW_new_final}
\end{align}

As in the case of SGD, a similar analysis can be carried out for any $\mu_0 \sim \Theta(E \left[  f(x_0) - f(x^*) \right])$. Furthermore, for PCM-RWT, $\mu_0$ can be tuned for each iteration in the same manner as described for PCM-SGD, that is, by decreasing it by a factor of $\gamma$ after every iteration.
This proof can be readily extended to interval of any length $l$, by changing the value of $i_0$ to $\log_2(\beta l/ \rho)$ and also appropriately changing the bound on $\E \left[  f(x_0) - f(x^*) \right]$ in~\eqref{eq:N_RW_new_final}. For a different minimizer, the sequence of subtrees $\cT_i$'s can be appropriately modified as described in~\cite{Wang2018} to obtain the same result. \\

Finally, using~\eqref{eq:N_RW_new_final}, we can obtain the bound on $\E[\tau(\epsilon)]$. If $M_{RW}$ denotes the random number of local tests carried out before termination then $\E[M_{RW}] \leq \E[3N_{RW-\text{new}}+ 3]$. Moreover, since the number of samples in each test can be at most $s_0(\rho)$, therefore, the expected number of samples can be no more than $\E[M_{RW}]s_0(\rho)$. Substituting the different bounds and the relation between $\rho$ and $\epsilon$, we obtain that for some constant $\tau_0 > 0$, independent of $\epsilon$
\begin{align}
    \E[\tau(\epsilon)] \leq \frac{\tau_0}{\epsilon} \log \left( \frac{\E \left[  f(x_0) - f(x^*) \right]}{\epsilon} \right)  \log^2 \left( \log \left( \frac{1}{\epsilon} \right)\right).
\end{align}

\subsection*{Regret in One CM Iteration}

In this section, we perform a brief analysis of the regret incurred in one CM iteration. This corresponds to the inner sum in the term $R_1$ in the regret decomposition of PCM, capturing the regret incurred by the routine $\upsilon$ in the local one dimensional minimization.
Let $x_{(m)}$ denote the sampling point at the $m^{\text{th}}$ time the local test is called by the random walk module and $\hat{T}_{m}$ denote the random number of samples taken at this point. Therefore, if $R_{RWT}(\epsilon)$ denotes the regret incurred by RWT in one CM iteration to get to precision of $\epsilon$, then we have,
\begin{align}
	{R}_{RWT}(\epsilon) & = \E \left[ \sum_{m = 1}^{M_{RW}} \sum_{t = 1}^{\hat{T}_{m}} F(x_{(m)}; \xi_t) - F(x^*, \xi_t)\right], \nonumber \\
	 & \leq \E \left[ \sum_{m = 1}^{M_{RW}} \sum_{t = 1}^{\hat{T}_{m}} \frac{1}{2 \alpha}[g(x_{(m)})]^2 \right], \nonumber \\
	& \leq \E \left[ \sum_{m = 1}^{M_{RW}} \E[\hat{T}_m] \frac{1}{2 \alpha}[g(x_{(m)})]^2 \right]. \label{eq:R_RWT_1}
\end{align}

Note that $\hat{T}_m$ is the random number of samples taken at sampling point $x_{(m)}$ with the termination rule. If $\tilde{T}$ denotes the random number of samples taken without the termination rule then, $\hat{T}_{m} = \tilde{T} \1\{ \tilde{T} \leq s_0(\rho) \}$ where $\rho = \sqrt{\alpha \epsilon/2}$. To bound $\E[\hat{T}_{m}]$, we use different methods depending on the gradient of the sampling point. If $|g(x_{(m)}| \leq \rho$, then we use the trivial bound $\E[\hat{T}_{m}] = \E[\tilde{T} \1\{ \tilde{T} \leq s_0(\rho) \}] \leq  s_0(\rho)$. For the other case of $|g(x_{(m)}| > \rho$, we note that $\displaystyle \E[\hat{T}_{m}] \leq  \E[\tilde{T}] \leq  \frac{40 \sigma_0^2}{g(x_{(m)})^2} \log \left( \frac{2}{\bp} \log \left( \frac{40  \sigma_0^2}{\bp g(x_{(m)})^2} \right)\right) + 2  \leq  \frac{40 \sigma_0^2}{g(x_{(m)})^2} \log \left( \frac{2}{\bp} \log \left( \frac{40  \sigma_0^2}{\bp \rho^2} \right)\right) + 2$. Plugging these bounds in~\eqref{eq:R_RWT_1}, we obtain,
\begin{align}
    {R}_{RWT}(\epsilon) 	& \leq \E \left[ \sum_{m = 1}^{M_{RW}} \frac{g(x_{(m)})^2}{2 \alpha} \left(\E[\hat{T}_m]  \1\{|g(x_{(m)})| > \rho \} + \E[\hat{T}_m] \1\{|g(x_{(m)})| \leq \rho \} \right) \right], \nonumber \\
	& \leq \E \bigg[ \sum_{m = 1}^{M_{RW}} \left\{ \frac{40\sigma_0^2}{[g(x_{(m)})]^2} \log \left( \frac{2}{\bp} \log \left( \frac{40  \sigma_0^2}{\bp \rho^2} \right)\right) + 2  \right\} \frac{1}{2 \alpha}[g(x_{(m)})]^2 \1\{|\E_{\xi}[G(x_{(m)}; \xi)]| > \rho \}  \nonumber  \\ 
	& \ \ \ + \frac{20 \sigma_0^2}{\rho^2} \log \left( \frac{2}{\bp} \log \left( \frac{40  \sigma_0^2}{\bp \rho^2} \right)\right) \frac{\rho^2}{2 \alpha}  \1\{|\E_{\xi}[G(x_{(m)}; \xi)]| \leq \rho \}  \bigg],\nonumber \\
	& \leq \E \bigg[ \sum_{m = 1}^{M_{RW}} \left\{ \frac{20 \sigma_0^2}{\alpha} \log \left( \frac{2}{\bp} \log \left( \frac{40  \sigma_0^2}{\bp \rho^2} \right)\right) + \frac{g_{\max}^2}{\alpha} \right\}  \1\{|\E_{\xi}[G(x_{(m)}; \xi)]| > \rho \}\nonumber \\
	\ \ \ & +  \frac{20 \sigma_0^2}{\alpha } \log \left( \frac{2}{\bp} \log \left( \frac{40  \sigma_0^2}{\bp \rho^2}  \right)\right)  \1\{|\E_{\xi}[G(x_{(m)}; \xi)]| \leq \rho \} \bigg], \nonumber\\
	& \leq   \left( \frac{20 \sigma_0^2}{\alpha} \log \left( \frac{2}{\bp} \log \left( \frac{40  \sigma_0^2}{\bp \rho^2} \right)\right) + \frac{g_{\max}^2}{\alpha} \right)  \E[M_{RW}], \nonumber\\ 
	& \leq  \left( \frac{20 \sigma_0^2}{\alpha} \log \left( \frac{2}{\bp} \log \left( \frac{80  \sigma_0^2}{\bp \alpha \epsilon} \right)\right) + \frac{g_{\max}^2}{\alpha} \right)  \E[3N_{RW-\text{new}}+ 3].
\end{align}
Substituting the bound from~\eqref{eq:N_RW_new_final} in the above equation, we can show that for some constant $\bar{R} > 0$, independent of $\epsilon$,  
\begin{align}
    {R}_{RWT}(\epsilon) \leq \bar{R} \log \left( \frac{\E \left[  f(x_0) - f(x^*) \right]}{\epsilon} \right)  \log^2 \left( \log \left( \frac{1}{\epsilon} \right)\right).
\end{align}

\subsection*{Regret Analysis of PCM-RWT}

We can now combine all the results obtained about performance of RWT to analyze the performance of PCM-RWT. Using the decomposition of regret in $R_1$ and $R_2$, we bound each of these terms individually to obtain the bound on the overall regret. \\

We begin with bounding $R_1$. Note that we can now rewrite $R_1$ as,
\begin{align}
    R_1 & \leq  \E \left[  \sum_{k = 1}^K  R_{RWT}(\epsilon_k) \right], \nonumber \\
    & \leq  \E \left[  \sum_{k = 1}^K \bar{R} \log \left(  \frac{ \E[ f(\x^{(k-1)}) - f(\x^*)] }{\epsilon_k} \right)  \log^2 \left( \log \left( \frac{1}{\epsilon_k} \right)\right)  \right], \nonumber\\
    & \leq  \E \left[  \sum_{k = 1}^K \bar{R} \log \left( \frac{ F_0 \gamma^k }{\epsilon_0 \gamma^k} \right)  \log^2 \left( \log \left( \frac{1}{\epsilon_0} \right) + k \log \left(\frac{1}{\gamma} \right)\right)  \right], \nonumber\\
    & \leq  \E \left[  \sum_{k = 1}^K  \bar{R}'  \log^2 \left(\log \left( \frac{1}{\epsilon_0} \right) + k \log \left(\frac{1}{\gamma} \right)\right)  \right].
\end{align}
Using the result from Lemma~\ref{lemma_k_lb} along with Jensen's inequality, we conclude that $R_1$ is of the order $O(\log T \log^2(\log T))$. Similarly, we now consider $R_2$.
\begin{align}
    R_2 &\leq \E \left[ \sum_{k = 1}^K \sum_{t = t_{k-1}+1}^{t_k} \left[ F(\x^*_{(i_k, \x^{(k-1)})}, \xi_t) - F(\x^*, \xi_t) \right]   \right], \nonumber \\
    & \leq \E \left[ \sum_{k = 1}^K [f(\x^{(k-1)} - f(\x^*)] \E[\tau(\epsilon_k)] \right], \nonumber \\
    & \leq \E \left[ \sum_{k = 1}^K (F_0 \gamma^{k-1}) \frac{\tau_0}{\epsilon_k} \log \left( \frac{ \E[ f(\x^{k-1}) - f(x^*)]}{\epsilon_k} \right)  \log^2 \left( \log \left( \frac{1}{\epsilon_k} \right)\right) \right], \nonumber \\
    & \leq \E \left[ \sum_{k = 1}^K \tau_0' \log^2 \left( \log \left( \frac{1}{\epsilon_0} \right) + k \log \left(\frac{1}{\gamma} \right)\right)  \right].
\end{align}
This is similar to the term we obtained in $R_1$ implying that $R_2$ is also of the order $O(\log T \log^2(\log T))$. Combining the two, we arrive at our required result.

\newpage

\section*{Appendix D}

In this section, we briefly describe the advantages obtained using parallelization. Consider the setup of $m$ cores, connected in parallel to the main server. To make it similar to our original setup, we assume that each processor has the access to the oracle independently of others. It is assumed that $m \leq d$. The algorithm for implementing PCM using parallel updates is described as follows
\begin{enumerate}
    \item Read the current iterate $\x$ and pass it to all cores.
    \item Select $m$ different indices from $\{1,2,\dots, d \}$ uniformly at random and allocate them to the cores.
    \item On each core, run the one dimensional optimization routine along the dimension whose was index assigned to that core. The initial point for all the cores will be the same point $\x$. Let the points returned by the cores to the server be denoted as $\y_1, \y_2, \dots \y_m$.
    \item Generate the next iterate $\displaystyle \x_1 = \frac{1}{m}\sum_{k = 1}^m \y_k$.
\end{enumerate}
The last step is the update or the synchronization step which ensures that the function value at the new iterate is lesser than that at the one previous one. Note that in the second step $m$ different indices are chosen uniformly at random, that is, one of the $\binom{d}{m}$ sets is chosen. \\

The analysis of the above mentioned parallel implementation scheme is very similar to that of the sequential case.
Let $\1_{(i,j)}$ denote the indicator variable for the $i^{\text{th}}$ direction and $j^{\text{th}}$ core. It is $1$ if the $i^{\text{th}}$ direction was chosen for optimization on the $j^{\text{th}}$ core where $i = 1, 2, \dots d$ and $j = 1, 2, \dots m$. Thus, from equation~\eqref{eq:det_CD_step}, we have that for each $j = 1,2, \dots m$, $\displaystyle \E[f(\y_k)| \x] \leq  f(\x) - \frac{1}{2\beta}\sum_{i = 1}^d \1_{(i,j)} [g_i(\x)]^2 + \epsilon $, where $\epsilon$ is the required accuracy. 
Using the update scheme in the synchronization step, we have,
\begin{align*}
	m \E[f(\x_1)| \x] & = m \E\left[ f \left( \frac{1}{m} \sum_{j = 1}^m \y_j \right) \bigg| \x \right] \\
	& \leq m  \left( \frac{1}{m} \sum_{j = 1}^m \E[f(\y_j) | \x] \right) \\
	& \leq \sum_{j = 1}^m \left( f(\x) - \frac{1}{2\beta}\sum_{i = 1}^d \1_{(i,j)} [g_i(\x)]^2 + \epsilon \right) \\
	\implies \E[f(\x_1) | \x] &  \leq \frac{1}{m} \left(\sum_{j = 1}^m \left( f(\x) - \frac{1}{2\beta}\sum_{i = 1}^d \1_{(i,j)} [g_i(\x)]^2 + \epsilon \right) \right)\\
	 &  \leq  f(\x) - \frac{1}{2m\beta}\sum_{j = 1}^m  \sum_{i = 1}^d \1_{(i,j)} [g_i(\x)]^2 + \epsilon \\
\end{align*}
where the second step follows from the convexity of the function.
Now taking expectation over the random choice of coordinates, we get,
\begin{align*}
	 \E[f(\x_1) | \x] & \leq  f(\x) - \E \left[ \frac{1}{2m \beta}\sum_{j = 1}^m  \sum_{i = 1}^d \1_{(i,j)} [g_{i}(\x)]^2 \right] + \epsilon \\
	 & \leq  f(\x)  -  \frac{1}{2m \beta}\sum_{j = 1}^m  \sum_{i = 1}^d \frac{m}{d}  [g_{i}(\x)]^2   + \epsilon \\
	 & \leq  f(\x)  -  \frac{m}{2d\beta} \sum_{i = 1}^d  [g_i(\x)]^2  + \epsilon \\
	 & \leq  f(\x)  -  \frac{m}{2d\beta} \| g(\x) \|^2  + \epsilon \\
	  & \leq  f(\x)  -  \frac{m \alpha}{d\beta} \left( f(\x) - f(\x^*) \right)  + \epsilon 
\end{align*}

Note that this expression is similar to one obtained in equation~\eqref{eq:det_CD_step}. Using an analysis similar to the one in Appendix A, we can obtain convergence rates for the case of parallel updates. The reduction in dimensionality dependence is evident through the factor $d$ being replaced by $d/m$.
